\documentclass[10pt]{article}

\usepackage[letterpaper,margin=1in]{geometry}
\usepackage{times}
\usepackage{natbib}
\usepackage[utf8]{inputenc}
\usepackage[T1]{fontenc}
\usepackage{hyperref}
\hypersetup{hidelinks}
\usepackage{url}
\usepackage{booktabs}
\usepackage{amsfonts}
\usepackage{nicefrac}
\usepackage{microtype}
\usepackage{xcolor}

\usepackage{graphicx}
\usepackage{subcaption}
\usepackage{amsmath}
\usepackage{bbm}
\usepackage{makecell}
\usepackage{multirow}
\usepackage[capitalize,noabbrev]{cleveref}
\usepackage{placeins}

\usepackage{amsthm}
\theoremstyle{plain}
\newtheorem{theorem}{Theorem}[section]
\newtheorem{proposition}[theorem]{Proposition}
\newtheorem{lemma}[theorem]{Lemma}
\newtheorem{corollary}[theorem]{Corollary}
\theoremstyle{definition}
\newtheorem{definition}[theorem]{Definition}
\newtheorem{assumption}[theorem]{Assumption}
\theoremstyle{remark}

\newcommand{\E}{\mathbb{E}}
\newcommand{\Prob}{\mathbb{P}}
\newcommand{\ind}{\mathbf{1}}
\newcommand{\norm}[1]{\left\lVert#1\right\rVert}
\newcommand{\F}{\mathsf{F}}

\title{A Structural Theory of Position Bias in Transformers}
\author{
Hanna Herasimchyk\thanks{Equal contribution.}\quad
Robin Labryga\footnotemark[1]\quad
Tomislav Prusina\footnotemark[1]\quad
S\"oren Laue\\[0.4em]
Department of Informatics, University of Hamburg, Hamburg, Germany
}
\date{}

\begin{document}

\maketitle


\begin{abstract}
Transformer models systematically favor certain token positions, yet the architectural origins of this position bias remain poorly understood. This bias is closely connected to the Lost-in-the-Middle phenomenon, where models underutilize information placed in the middle of the context. We show that causal Transformer architectures induce a systematic position-bias prior that can help explain the influence profiles associated with Lost-in-the-Middle. To characterize this prior, we develop a structural theory based on cumulative attention rollout with layerwise residual mixing. At finite depth, we isolate the competing positional drifts induced by causal masking, residual mixing, and positional encoding. Their interaction produces broad, often U-shaped profiles. At infinite depth, prior work predicts inevitable collapse under attention-only theory and identifies the discrepancy with observed Transformer behavior as an open problem. Under the stated boundedness assumptions, our residual-aware analysis replaces inevitable collapse with a dichotomy governed by cumulative attention mixing. Empirically, across 30 model-dataset settings, our adaptive residual-aware rollout closely matches measured input-token influence and reduces the average profile discrepancy by approximately $85\%$ relative to both the attention-only and fixed-residual rollout baselines.
\end{abstract}


\section{Introduction}
\label{sec:introduction}

Despite their widespread success, the mechanisms by which Transformers propagate information across depth remain poorly understood.
In particular, modern Transformer models exhibit strong and systematic position biases that persist across tasks, inputs, and even lexical scrambling~\citep{rahimi2601}.
These include primacy bias (preference for early tokens), recency bias (preference for recent tokens), and the Lost-in-the-Middle (LiM) phenomenon, where information in the center of long contexts is underutilized~\citep{liu2307}.
A growing body of empirical evidence suggests that these effects are not primarily semantic, but instead reflect architectural structure.

Recent theoretical work has sought to explain position bias through attention rollout, interpreting stacked attention layers as a stochastic information propagation process across tokens.
In a recent fundamental analysis, \citet{wu2502} show that, in attention-only causal Transformers, cumulative attention must collapse onto the first token as depth increases.
This asymptotic prediction is at odds with the non-collapsed profiles observed
in deep pretrained Transformers, a theory-observation gap that \citet{wu2502}
explicitly identify as an open problem.

Existing rollout models either omit the residual path or mix it with attention
at a fixed one-half weight. We instead model their relative strength with an
adaptive layerwise mixing schedule estimated from forward branch geometry.
We analyze the resulting cumulative rollout and isolate three architectural forces at finite depth: causal masking induces primacy, stronger residual preservation induces recency, and any positional mechanism that shifts its transport kernels toward recent keys strengthens recency.
In the models studied, these forces combine into broad, often U-shaped influence profiles, forming a structural prior consistent with Lost-in-the-Middle.
At infinite depth, the cumulative mixing schedule determines whether rollout retains mass at later tokens or collapses onto the first token.

Our analysis compares attention-only and fixed-residual rollouts with our adaptive residual-aware rollout, which yields broad, typically U-shaped profiles.
Comparison to empirical input-output token influence in pretrained language models shows close agreement, providing an architecture-faithful framework for understanding position bias.

\newpage
\noindent Our contributions are as follows:
\begin{itemize}\setlength{\itemsep}{0.5em}
  \item \textbf{Adaptive residual-aware attention rollout.}
  We model identity and attention paths through an adaptive layerwise mixing
  schedule, estimated from forward branch geometry without using the influence
  target.

  \item \textbf{Architectural origin of position bias.}
  We show that the architectural transport mechanisms studied here produce broad, often U-shaped influence profiles across ALiBi- and RoPE-based models, establishing a structural prior consistent with empirically observed position-dependent effects.

  \item \textbf{Resolution of the analytical open problem \citep{wu2502}.}
  Within the residual-aware transport model, we generalize the prior inevitability result to a dichotomy: under the stated boundedness assumptions, cumulative attention mixing determines whether collapse occurs.

  \item \textbf{Empirical validation against measured token influence.}
  Across 30 model-dataset settings, our rollout reduces the mean Wasserstein
  distance to measured token influence from $0.62$ for Wu et al. to
  $0.09$, an $85\%$ reduction, while also outperforming fixed residual mixing.
\end{itemize}


\section{Related work}
\label{sec:related_work}

\noindent\textbf{Attention rollout and theoretical position bias.}
Attention rollout and information-flow perspectives have long been used to analyze Transformer behavior \citep{abnar2005,chefer2021,barbero2406}. Most closely related to our work, \citet{wu2502} provide a rigorous analysis of attention-only causal Transformers and prove that cumulative attention must collapse onto the first token as depth increases. They also examine residual effects empirically, but their collapse theorem uses an attention-only propagation model. Empirically, such collapse is not observed in modern large-scale Transformers, motivating a theory that tracks effective residual mixing across layers.
Concurrent work derives closed forms for a uniform-Ces\`aro, constant-mixing
specialization of residual rollout, which also exhibits fixed-length
infinite-depth primacy collapse~\citep{chowdhury2026middle};
Appendix~\ref{sec:approx} gives the precise comparison.
In a simplified high-dimensional model, \citet{kim2509} show that LayerNorm induces recency in second-layer attention scores, including in the presence of residual connections; their target is local attention-score ordering rather than cumulative input-output influence across depth.
ALTI shows that raw attention weights can miss value- and output-projection
effects by constructing vector-aware token contributions~\citep{ferrando2022alti}.
Our RoPE kernel uses an ALTI-style scalarization for corpus-averaged transport,
while our target is systematic position bias rather than input-specific
attribution.

\noindent\textbf{Empirical position bias.}
A substantial body of empirical work documents systematic position bias in Transformer attention, including primacy bias or ``attention sinks'' at early tokens \citep{zheng2306,xiao2309,yin2402,guo2410,kaul2410,cobbina2507} and recency bias favoring late tokens \citep{zhao2102,sun2109,he2302,liu2307}.
These effects are robust across inputs and tasks, persist under lexical scrambling \citep{rahimi2601}, vary across positional encodings such as RoPE \citep{su2104} and ALiBi \citep{press2108,kazemnejad2305,yang2501}, and arise even without explicit encodings \citep{guo2410}, suggesting architectural rather than semantic origins.
Building on the multilayer account of causal masking and positional encodings in \citet{wu2502}, we characterize how adaptive, layer-dependent residual mixing changes these dynamics at finite and infinite depth. The resulting corpus-level rollouts closely match gradient and blank-out token influence in pretrained Transformers and substantially improve upon both the attention-only and fixed-residual rollout baselines.

\noindent\textbf{Lost-in-the-Middle phenomenon.}
The coexistence of primacy and recency biases gives rise to the Lost-in-the-Middle effect, where model performance degrades for information located in the center of long contexts \citep{liu2307,guo2406,menschikov2505}. Prior work characterizes this effect empirically and offers functional interpretations \citep{veseli2508,salvatore2510}.
Several approaches aim to control such biases through architectural modifications or alternative attention mechanisms by reformulating softmax \citep{kaul2410} or introducing multi-scale positional encodings \citep{zhang2403}.
While effective in practice, neither line of work provides a unified architectural account of why U-shaped profiles arise or why such interventions are necessary. Our theory offers this principled perspective and complements existing mitigations.


\section{Notation and problem setup}
\label{sec:setup}

We introduce notation and define a theory for analyzing how architectural components
of Transformer models shape cumulative cross-token information propagation across
layers, without modeling training objectives or optimization dynamics.
Our framework is analytically tractable yet close to practical architectures. It explicitly models cross-token transport, residual connections, attention masking, and relative positional encodings, while abstracting away token-wise components such as MLP blocks and layer normalization.

\noindent\textbf{Tokens and layers.}
We consider a sequence of $n$ tokens indexed by $i,j\in\{1,\ldots,n\}$ and a Transformer
of depth $T$, with layers $t\in\{1,\ldots,T\}$.

\noindent\textbf{Masks.}
Attention is constrained by a binary mask $\mathcal{M}\subseteq\{1,\ldots,n\}^2$,
where $(i,j)\in\mathcal{M}$ indicates that token $i$ may attend to token $j$.
We consider causal masking: $(i,j)\in\mathcal{M}$ iff $j\le i$.

At layer $t$, let $A^{(t)}\in\mathbb{R}^{n\times n}$ be a nonnegative,
row-stochastic matrix supported on $\mathcal M$, representing cross-token
propagation.
We instantiate this kernel for ALiBi and RoPE as follows.

\noindent\textbf{ALiBi kernel.}
For ALiBi, the kernel is the attention matrix, computed via a masked row-wise softmax
\[
A^{(t)}_{i\cdot} = \mathrm{softmax}(\ell^{(t)}_{i\cdot}), \quad \ell^{(t)}_{ij} =
\begin{cases}
s^{(t)}_{ij} + b^{(t)}_{ij} & \text{if } (i,j)\in\mathcal{M},\\
-\infty & \text{otherwise},
\end{cases}
\]
where $b^{(t)}_{ij}$ captures position-dependent biases and $s^{(t)}_{ij}$ captures content-dependent contributions. ALiBi uses $b^{(t)}_{ij}=-m(i-j)$ with positional slope $m\ge0$.
This expression uses single-head notation.

\noindent\textbf{RoPE kernel.}
For RoPE, we likewise write $A^{(t)}_{i\cdot}=\mathrm{softmax}(\ell^{(t)}_{i\cdot})$.
Here $\ell^{(t)}_{ij}$ denotes an effective transport logit incorporating the raw attention, which already includes RoPE's positional information, and the value and output projections.
For $(i,j)\notin\mathcal M$, set $\ell^{(t)}_{ij}=-\infty$.
The effective logits are specified in \cref{app:kernel-constructions}.

\noindent\textbf{Residual connections and mixing.}
Each attention layer is embedded in a residual block
\[
X^{(t)} \;=\; X^{(t-1)} + \mathrm{Attn}(X^{(t-1)}),
\]
combining identity propagation with an attention-driven update.
At the level of cross-token propagation, we represent the attention-driven
update by the row-stochastic transport kernel $A^{(t)}$ defined above. For
ALiBi this is the attention matrix; for RoPE the effective kernel also accounts
for the value and output projections.

\noindent\textbf{Residual-aware transition with adaptive mixing.}
Existing formulations account for residuals by augmenting attention with identity and renormalizing; e.g., \citet{abnar2005} replace $A^{(t)}$ with a fixed convex combination $\tfrac12 I + \tfrac12 A^{(t)}$, where $I\in\mathbb R^{n\times n}$ is the identity matrix.
Residual networks also admit an interpretation as forward-Euler discretizations
of continuous-time dynamics~\citep{haber2017ode,chen2018neural}. Since numerical
step sizes may vary across steps~\citep{hairer1993solving}, this viewpoint
motivates layer- and input-dependent attention strength rather than a fixed
one-half weight.
Since our analysis concerns only the relative contribution of the identity and attention paths, let $\lambda_t\in[0,1]$ denote the attention-branch weight.
The resulting convex combination induces the row-stochastic residual-aware transition matrix
\begin{equation}
\label{eq:R-def}
R^{(t)} \;:=\; (1-\lambda_t)I + \lambda_t A^{(t)} .
\end{equation}
The theory allows any specified schedule $\{\lambda_t\}_{t=1}^T$; its empirical
construction is given in \cref{sec:methodology}.

\noindent\textbf{Cumulative rollout.}
Cumulative information propagation to depth $T$ is captured by the rollout matrix
\[
P^{(T)} := R^{(T)}R^{(T-1)}\cdots R^{(1)}.
\]
The entry $P^{(T)}_{ij}$ quantifies the rollout influence of input token $j$ on token $i$ within the transport model.
Each $R^{(t)}$ is row-stochastic; hence $P^{(T)}$ is row-stochastic as well.
Under causal masking, the last-row distribution $p^{(T)}(j) := P^{(T)}_{nj}$ describes this influence on the final position used to predict the next token.
This distribution is the central object: we analyze how attention masking, positional encodings, and the residual-mixing schedule shape $p^{(T)}$ at finite depth and whether it collapses in the infinite-depth limit.


\section{Finite-depth attention dynamics: three forces}
\label{sec:finite-depth}
To characterize systematic architectural effects at finite depth that shape the rollout distribution $p^{(T)}$, our analysis
relies on the following weak positional monotonicity assumption on the
transport kernels, reflecting the ordering induced by causal masking while remaining agnostic to
content and parameter values.

\begin{assumption}[Stochastically monotone transport kernels]
\label{ass:stoch-monotone-A}
For each layer $t\in\{1,\dots,T\}$, let $A^{(t)}\in\mathbb{R}^{n\times n}$ denote the transport kernel.
We assume that $A^{(t)}$ is stochastically monotone~(Definition~\ref{def:stoch-mon}) with
respect to the natural order on positions: for all $1\le i<i'\le n$ and all
prefix cutoffs $k\in\{1,\dots,n\}$,
\[
\sum_{j=1}^{k} A^{(t)}_{ij}
\;\ge\;
\sum_{j=1}^{k} A^{(t)}_{i'j}.
\]
Equivalently, for $i<i'$, the row distribution $A^{(t)}_{i,\cdot}$ is
first-order stochastically smaller than $A^{(t)}_{i',\cdot}$.
\end{assumption}

This condition is naturally compatible with a broad class of transport kernels
induced by causal masking and smoothly decaying or symmetric kernels.
It is especially natural when $A^{(t)}$ is interpreted as an effective
(head-aggregated and corpus-averaged) kernel that abstracts away content-dependent
fluctuations and captures net architectural tendencies.

For the corpus-averaged ALiBi kernels, strict prefix-monotonicity is not always
satisfied exactly, but deviations are numerically negligible.
Across models, layers, and datasets, violations are extremely rare,
occurring for at most $10^{-6}$ of all triples $(i<i',k)$.
The mean prefix-mass shortfall conditional on a violation is approximately
$2 \times 10^{-6}$,
yielding an unconditional expected shortfall on the order of $10^{-12}$.
For the projection-aware RoPE kernels, violations occur for approximately
$0.42\%$ of all triples. The mean prefix-mass shortfall conditional on a
violation is approximately $8.3 \times 10^{-5}$, yielding an unconditional
expected shortfall of approximately $3.5 \times 10^{-7}$.
These results indicate that $A^{(t)}$ is extremely close to
prefix-monotone in practice, supporting Assumption~\ref{ass:stoch-monotone-A}
as a stable regularity condition for the finite-depth force decomposition.

Under this regularity condition, the following propositions isolate three
architectural forces represented in our transport model.

\noindent\textbf{Causal masking induces primacy drift.}
When composed across layers, the directional asymmetry of causal masking accumulates into a drift toward earlier positions in the rollout.

\begin{proposition}[Primacy drift under causal masking]
\label{prop:causal-primacy}
Assume causal masking and Assumption~\ref{ass:stoch-monotone-A}.
Let $p^{(t)}$ for $t\in\{1,\dots,T-1\}$ denote the last-row rollout distribution
$p^{(t)}(j)=P^{(t)}_{nj}$.
Then for every prefix length $k<n$,
the prefix mass is monotonically increasing:
\[
\sum_{j=1}^{k} p^{(t+1)}(j)\;\ge\;\sum_{j=1}^{k} p^{(t)}(j).
\]
Consequently, at finite depth the last-row rollout distribution exhibits a
systematic drift of mass toward earlier positions.
\end{proposition}

\noindent\textbf{Residual mixing controls recency drift.}
Decreasing $\lambda_t$ strengthens identity propagation and reduces cross-token
mixing. Information is then retained more strongly at its current position
rather than propagated backward, biasing finite-depth rollout toward recent
tokens.

\begin{proposition}[Residual strength induces recency drift]
\label{prop:residual-recency}
Assume causal masking and Assumption~\ref{ass:stoch-monotone-A}.
Fix transport kernels $\{A^{(t)}\}_{t=1}^T$ and consider two residual-mixing schedules $\{\lambda_t\}$ and
$\{\lambda'_t\}$ with $\lambda_t \le \lambda'_t$ for all $t$.
Let $p^{(T)}$ and ${p'}^{(T)}$ denote the corresponding last-row rollout
distributions.
Then for every prefix length $k<n$,
\[
\sum_{j=1}^{k} p^{(T)}(j)
\;\le\;
\sum_{j=1}^{k} {p'}^{(T)}(j).
\]
Equivalently, strengthening residual connections (i.e., decreasing
$\lambda_t$) induces a recency drift in the last-row rollout distribution.
\end{proposition}

\noindent\textbf{Positional encodings induce recency drift.}
When positional encodings systematically prefer recent keys, they induce a recency bias
that compounds across layers through rollout.

\begin{proposition}[Positional encodings induce recency drift]
\label{prop:pe-vs-nope}
Assume causal masking. For each layer $t$, let
\[
A^{(t)}_0 := \mathrm{softmax}(s^{(t)}) ,
\qquad
A^{(t)}_{\mathrm{PE}} := \mathrm{softmax}(s^{(t)} + b^{(t)}),
\]
where positional logits $b^{(t)}_{ij}$ are recency-favoring in the sense that, for all admissible keys $j<k$,
\[
b^{(t)}_{ij} \le b^{(t)}_{ik}.
\]
Assume that $\{A^{(t)}_0\}_{t=1}^T$ satisfies
Assumption~\ref{ass:stoch-monotone-A}.

Let
\[
R^{(t)}_0 := (1-\lambda_t)I + \lambda_t A^{(t)}_0,
\qquad
R^{(t)}_{\mathrm{PE}} := (1-\lambda_t)I + \lambda_t A^{(t)}_{\mathrm{PE}},
\]
and denote the corresponding rollout distributions by
$p^{(T)}_0$ and $p^{(T)}_{\mathrm{PE}}$.

Then for every prefix length $k<n$,
\[
\sum_{j=1}^k p^{(T)}_{\mathrm{PE}}(j)
\;\le\;
\sum_{j=1}^k p^{(T)}_0(j).
\]
Equivalently, positional encodings induce a recency drift in the last-row rollout distribution.
\end{proposition}

\noindent\textbf{Example (ALiBi).}
Under causal masking, ALiBi assigns larger biases to more recent
keys for each query position and therefore satisfies the recency-favoring
condition of Proposition~\ref{prop:pe-vs-nope}. When the baseline kernels satisfy
Assumption~\ref{ass:stoch-monotone-A}, ALiBi consequently induces a recency drift
in the final rollout distribution.

\noindent\textbf{Structural forces and Lost-in-the-Middle.}
The propositions identify competing architectural drifts at finite depth.
In the models studied below, their combination yields broad, non-collapsed
profiles in which early and recent tokens receive more influence than
intermediate positions. This gives a structural account of the U-shaped
influence prior associated with the Lost-in-the-Middle phenomenon, without
claiming that every admissible kernel and mixing schedule must be U-shaped.
Appendix~\ref{sec:shape} gives explicit sufficient conditions and a computable
finite-mixing certificate under the common-kernel approximation.

\section{Infinite-depth attention dynamics}
\label{sec:infinite-depth}
We now turn from finite- to infinite-depth attention dynamics.
For the attention mechanism and its boundedness conditions, we follow the setup
of \citet{wu2502}. We additionally allow bounded additive positional-bias
logits. They show that attention-only causal rollout collapses to the first
token in the infinite-depth limit. We study the corresponding residual-aware
dynamics, with their attention-only setting recovered by taking
$\lambda_t \equiv 1$.

Let $W_Q^{(t)}$, $W_K^{(t)}$, and
$W_V^{(t)}$ denote the query, key, and value projections. For admissible keys,
the content logits are
$s_{ij}^{(t)}=(X_{i,:}^{(t-1)}W_Q^{(t)})(X_{j,:}^{(t-1)}W_K^{(t)})^\top$,
and $A^{(t)}$ is the causal row-wise softmax of $s^{(t)}+b^{(t)}$. The
normalized dynamics are
\[
X^{(t)}=(1-\lambda_t)X^{(t-1)}
+\lambda_t A^{(t)}X^{(t-1)}W_V^{(t)}.
\]
Without additive positional-bias logits, the bounded-bias assumption below is
trivially satisfied and the collapse branch of our theorem reduces to their
result.

\begin{assumption}[Bounded query and key operators]
\label{ass:bounded_qk}
As in \citet{wu2502}, we assume there exists a constant $C>0$ such that
$\|W_Q^{(t)}\|_2\le C$ and $\|W_K^{(t)}\|_2\le C$ for all $t\ge1$, where
$\|\cdot\|_2$ is the operator norm.
\end{assumption}

\begin{assumption}[Bounded value accumulation]
\label{ass:bounded_v}
We assume the sequence of value transformations is well-behaved, such that the layer-wise operator norms $C_t := \|W_V^{(t)}\|_2$ satisfy
$\sum_{t=1}^{\infty} \lambda_t \max(0, C_t - 1) < \infty$.
This is the residual-aware analogue of the bounded value-accumulation condition used by \citet{wu2502}. Because residual connections introduce skip paths that bypass contiguous prefix products, stability requires constraining the excess growth of the layer-wise transformations directly.
\end{assumption}

\begin{assumption}[Bounded positional-bias range]
\label{ass:bounded_b}
When additive positional-bias logits $b_{ij}^{(t)}$ are present, we assume that their admissible row-wise range is uniformly bounded: there exists $B < \infty$ such that
$|b^{(t)}_{ij}-b^{(t)}_{ik}|\le B$ for all layers $t$, query positions $i$,
and admissible keys $j,k$.
\end{assumption}

\begin{lemma}[Uniform stability and attention lower bound]
\label{lem:bound_epsilon}
Suppose Assumptions~\ref{ass:bounded_qk}-\ref{ass:bounded_b} hold and the initial input is bounded such that $\|X^{(0)}\|_2 \le C$ under causal masking.
Then there exist constants $C_X > 0$ and $\varepsilon > 0$, both independent of the depth $t$, such that for all $t \ge 1$:
\begin{enumerate}\itemsep0pt\parskip0pt\parsep0pt
    \item[\textnormal{(i)}] The hidden states are uniformly bounded: $\max_{i} \|X_{i,:}^{(t)}\|_2 \le C_X$.
    \item[\textnormal{(ii)}] The attention weights are uniformly bounded away from zero: $A_{ij}^{(t)} \ge \varepsilon$ for all $(i, j) \in \mathcal{M}$.
\end{enumerate}
\end{lemma}

\begin{theorem}[Residual-aware infinite-depth collapse dichotomy]
\label{thm:residual_dichotomy}
Under causal masking and Assumptions~\ref{ass:bounded_qk}-\ref{ass:bounded_b}, the asymptotic behavior of the residual-aware attention rollout
$P^{(T)} := R^{(T)}\cdots R^{(1)}$ is fully determined by the summability of the
mixing coefficients $\{\lambda_t\}_{t\ge1}$.

\begin{enumerate}\itemsep0pt\parskip0pt\parsep0pt
\item[(i)] \textbf{Finite total mixing $\Rightarrow$ no collapse.}
If $\sum_{t=1}^\infty \lambda_t < \infty$, then for every token $i \in \{1,\dots,n\}$,
\[
\liminf_{T \to \infty} P^{(T)}_{ii}
\;\ge\;
\prod_{t=1}^{\infty} \bigl(1-(1-\varepsilon)\lambda_t\bigr)
\;>\; 0.
\]
Thus, for every $i>1$, the rollout does not collapse to the first token.

\item[(ii)] \textbf{Infinite total mixing $\Rightarrow$ collapse.}
If $\sum_{t=1}^\infty \lambda_t = \infty$, then for every token $i \in \{1,\dots,n\}$,
$\lim_{T \to \infty} P^{(T)}_{i1} = 1$.
Moreover, for all $1 < j \le i$ and all $T \ge 1$, there exists a constant $\varepsilon>0$ such that
\[
P^{(T)}_{ij}
\;\le\;
\exp\!\Bigl( - (j-1)\varepsilon \sum_{s=1}^{T} \lambda_s \Bigr),
\]
and consequently $P^{(T)}_{ij} \to 0$ for all $j>1$ as $T \to \infty$.
\end{enumerate}
\end{theorem}

\noindent\textbf{Discussion.}
Thus, in the asymptotic analytical setting of \citet{wu2502}, cumulative
residual mixing, rather than depth alone, governs collapse: fixed positive
mixing still collapses. The depth-uniform lower bound $\varepsilon>0$ supports
both the quantitative bounds and the collapse conclusion. The theorem neither
establishes that measured
finite-model schedules are summable nor extends its positive-attention bound to
empirical projection-aware RoPE kernels. Its finite-$T$ bounds quantify how
cumulative mixing controls the approach to the asymptotic limit.

\noindent\textbf{Multi-head extension.}
We model the effective attention operator at layer $t$ as the uniform head average over the $H$ causal row-stochastic head kernels $A^{(t,h)}$, $A^{(t)} := \tfrac{1}{H}\sum_{h=1}^H A^{(t,h)}$.
Since prefix masses and residual-aware transitions depend
linearly on the underlying attention kernel, all finite- and infinite-depth
results established for a single head apply verbatim to $A^{(t)}$.


\section{Empirical validation}
\label{sec:experiments}
After describing the experimental setup, we construct corpus-level rollouts and measure residual mixing schedules in pretrained LLMs.
We then compare three rollout rules using gradient-based input-token influence:
(a) the attention-only rollout of \citet{wu2502}, with $\lambda_t=1$;
(b) the fixed residual augmentation of \citet{abnar2005}, with
$\lambda_t=\tfrac12$; and (c) our adaptive residual-aware rollout, with a
measured layerwise $\lambda_t$.

\noindent\textbf{Experimental setup.}
We evaluate a suite of pretrained ALiBi and RoPE-based language models summarized in \cref{tab:llms-used}. The ALiBi models comprise \texttt{BLOOM}~\citep{bigscience2205}, \texttt{MPT}~\citep{mosaicml2305}, and \texttt{Falcon}~\citep{penedo2306}; the RoPE-based models comprise \texttt{Qwen2.5}~\citep{qwenteam2409}, \texttt{Qwen3}~\citep{qwenteam2505}, and the hybrid RoPE/NoPE \texttt{SmolLM3}~\citep{huggingfacetb2507}.
At sequence length $n=256$, both the ALiBi and RoPE results average over 1{,}000 FineWeb-Edu prompts~\citep{dataset-finewebedu}.
Appendix~\ref{app:rollout} provides the available results on DCLM-Baseline and
Wikipedia~\citep{dataset-dclm,dataset-wikipedia}; Appendix~\ref{app:ablations}
provides longer-context results at $n=2048$.
The complete code is available here\footnote{\url{https://github.com/ml-uhh/position-bias}}.

\subsection{Corpus-level rollout construction}
\label{sec:methodology}
Our empirical target is the systematic input-token influence profile of a
pretrained Transformer across a corpus. Let $\mathcal D$ denote the
evaluation-corpus distribution and allow the kernel and mixing coefficient to
depend on an input $x$, with
$R^{(t)}(x)=(1-\lambda_t(x))I+\lambda_t(x)A^{(t)}(x)$. The exact corpus-level
prediction of the transport model is
\begin{equation}
\bar P^{(T)}
:=\mathbb E_{x\sim\mathcal D}\!\left[
R^{(T)}(x)\cdots R^{(1)}(x)\right],
\qquad
\bar p^{(T)}=e_n^\top\bar P^{(T)}.
\label{eq:corpus-rollout}
\end{equation}
The influence profile is not directly observable; we use gradient influence as
the main empirical proxy and validate it with token blank-out in
Appendix~\ref{app:blankout}.

We approximate the exact corpus-level transport prediction in \cref{eq:corpus-rollout} in
two steps. First, we interchange the expectation of the ordered
product with the ordered product of expectations:
\begin{equation}
\mathbb E_{x\sim\mathcal D}\!\left[R^{(T)}(x)\cdots R^{(1)}(x)\right]
\approx
\mathbb E_{x\sim\mathcal D}[R^{(T)}(x)]\cdots
\mathbb E_{x\sim\mathcal D}[R^{(1)}(x)].
\label{eq:product-of-expectations}
\end{equation}
This neglects cross-layer dependence between the input-dependent transitions
while retaining the average transition at each layer.

As a second approximation, we apply a mean-field factorization within each
layer. Writing $\bar\lambda_t=\mathbb E_{x\sim\mathcal D}[\lambda_t(x)]$ and
$\bar A^{(t)}=\mathbb E_{x\sim\mathcal D}[A^{(t)}(x)]$, we replace
$\mathbb E_{x\sim\mathcal D}[\lambda_t(x)A^{(t)}(x)]$ with
$\bar\lambda_t\bar A^{(t)}$. This gives
\begin{equation}
\mathbb E_{x\sim\mathcal D}[R^{(T)}(x)]\cdots
\mathbb E_{x\sim\mathcal D}[R^{(1)}(x)]
\approx P_{\mathrm{mf}}^{(T)}
:=\left[(1-\bar\lambda_T)I+\bar\lambda_T\bar A^{(T)}\right]
\cdots
\left[(1-\bar\lambda_1)I+\bar\lambda_1\bar A^{(1)}\right].
\label{eq:mean-field-rollout}
\end{equation}
Thus each expected transition is represented using the separately averaged
mixing coefficient and transport kernel. Appendix~\ref{sec:approx} distinguishes
the two approximations and develops further rollout simplifications.
Appendix~\ref{app:mean-field-ablation} evaluates them separately. Across all
tested models and corpora, neither introduces more than $6.5\times10^{-3}$
normalized final-row $1$-Wasserstein error, well below the
distances between rollout and measurement reported below.
The resulting corpus-level model prediction is
$\bar p_{\mathrm{mf}}^{(T)}=e_n^\top P_{\mathrm{mf}}^{(T)}$.

This construction is designed to isolate architectural transport rather than
reproduce corpus-level influence through learned content associations. We keep
the model's depth, causal mask, positional mechanism, and residual mixing, while
removing input-dependent content as far as the positional mechanism permits.
The residual mixing schedule is still measured from the pretrained model, as
defined next, but it is estimated independently of the influence profiles used
for evaluation. The two approximations above are applied after choosing the
transport kernels. Removing the ALiBi content logits below is a controlled
architectural intervention, not another mean-field approximation.

For each input and layer, we measure the attention-branch weight
\begin{equation}
\lambda_t(x)
=
\frac{\|X^{(t)}\|+\|\mathrm{Attn}(X^{(t-1)})\|-\|X^{(t-1)}\|}
{2\|X^{(t)}\|},
\label{eq:shapley}
\end{equation}
using Frobenius norms over tokens and features. This two-branch Shapley
allocation~\citep{shapley1953value}, applied to the coalition value
$v(S)=\|\sum_{b\in S}b\|$, accounts for reinforcement and cancellation between the residual
stream and attention update. No influence measurement enters this estimate.
For nonzero $\|X^{(t)}\|$, the triangle and reverse-triangle inequalities give
$0\leq\lambda_t(x)\leq1$, as required by the transport model.
In practice, we estimate all corpus expectations by empirical averages.

Appendix~\ref{app:lambda} reports the measured norm-Shapley schedules for both
model families; Appendix~\ref{app:mixing-ablation} evaluates sensitivity to
the alternative relative-norm calibration.

It remains to construct the kernels $\bar A^{(t)}$. All kernels match the depth,
masking pattern, and positional encoding of the corresponding pretrained model.
For ALiBi, we can perform this isolation directly: we remove content by setting
$s_{ij}^{(t)}=0$ in the attention logits. The resulting kernels retain only the
causal mask and the model's fixed positional slopes, and therefore contain no
input-dependent content.

The same intervention is unavailable for RoPE: setting its content logits to
zero would also remove the rotary positional signal. We therefore construct
projection-aware forward-contribution kernels and average them over the corpus.
These kernels retain attention values and head-specific output projections but
use neither gradients nor blank-out measurements; Appendix~\ref{app:kernel-constructions}
gives the construction. This yields a strongly compressed, corpus-averaged
transport model that preserves RoPE's positional signal.

We compute $\bar p_{\mathrm{mf}}^{(T)}$ using the same kernels under three
mixing rules. Variant (a), following \citet{wu2502}, sets
$\bar\lambda_t=1$ and removes the residual path. Variant (b), following the
residual-augmented rollout of \citet{abnar2005}, sets
$\bar\lambda_t=\tfrac12$. Variant (c), ours, uses the adaptive measured
schedule $\{\bar\lambda_t\}_{t=1}^T$. The comparison in
\cref{fig:rollout_comparison_main} shows variants (a)--(c)
alongside measured influence.

\begin{table}[t]
\small
\begin{minipage}[t]{0.47\textwidth}
\centering
\caption{Causal Transformer models included in the experimental study.}
\label{tab:llms-used}
\setlength{\tabcolsep}{2.25pt}
\begin{tabular}{lcrc}
\toprule
\makecell{\strut\\Model} & \makecell{\strut\\Encoding} & \makecell{\strut\\Params.} & \makecell{\strut\\Depth} \\
\midrule
\texttt{falcon-rw-7b} & ALiBi & 7B & 36 \\
\texttt{mpt-7b}       & ALiBi & 7B & 32 \\
\texttt{mpt-30b}      & ALiBi & 30B & 48 \\
\texttt{bloom-7b1}    & ALiBi & 7.1B & 30 \\
\texttt{bloom-176b}   & ALiBi & 176B & 70 \\
\texttt{qwen2.5-7b}   & RoPE & 7B & 28 \\
\texttt{qwen2.5-32b}  & RoPE & 32B & 64 \\
\texttt{qwen2.5-72b}  & RoPE & 72B & 80 \\
\texttt{qwen3-8b-base} & RoPE & 8B & 36 \\
\texttt{smollm3-3b-base} & RoPE & 3B & 36 \\
\bottomrule
\end{tabular}
\end{minipage}%
\hfill
\begin{minipage}[t]{0.51\textwidth}
\centering
\captionsetup{width=0.95\linewidth}
\caption{Normalized $1$-Wasserstein distances between rollout profiles and gradient influence.}
\label{tab:wasserstein_large}
\setlength{\tabcolsep}{2.5pt}
\begin{tabular}{lccc}
\toprule
\makecell{\strut\\Model} &
\makecell{(a) Wu\\$\lambda=1$} &
\makecell{(b) Abnar\\$\lambda=\tfrac12$} &
\makecell{(c) Ours\\adaptive} \\
\midrule
\texttt{falcon-rw-7b} & 0.66 & 0.63 & \textbf{0.03} \\
\texttt{mpt-7b}   & 0.62 & 0.58 & \textbf{0.15} \\
\texttt{mpt-30b}  & 0.61 & 0.61 & \textbf{0.02} \\
\texttt{bloom-7b1} & 0.63 & 0.58 & \textbf{0.13} \\
\texttt{bloom-176b} & 0.65 & 0.65 & \textbf{0.11} \\
\texttt{qwen2.5-7b} & 0.63 & 0.59 & \textbf{0.12} \\
\texttt{qwen2.5-32b} & 0.59 & 0.59 & \textbf{0.11} \\
\texttt{qwen2.5-72b} & 0.60 & 0.60 & \textbf{0.12} \\
\texttt{qwen3-8b-base} & 0.58 & 0.57 & \textbf{0.09} \\
\texttt{smollm3-3b-base} & 0.60 & 0.60 & \textbf{0.11} \\
\bottomrule
\end{tabular}
\end{minipage}
\end{table}

\subsection{Gradient input-token influence}
\label{subsec:input-token-influence}

For an input $x$ with embeddings $\mathbf e$ and predicted next token
$y(x)=\arg\max_v z_v(x)$, the local Jacobian-norm score is
\[
g_j(x)=\left\lVert\nabla_{\mathbf e_j}P(y(x)\mid x)\right\rVert_2
\]
\citep{simonyan1312,liu2601}. We average raw scores across the corpus and then
normalize across positions:
\begin{equation}
q_{\mathrm{grad}}(j)
=\frac{\mathbb E_{x\sim\mathcal D}[g_j(x)]}
{\sum_k\mathbb E_{x\sim\mathcal D}[g_k(x)]}.
\label{eq:input_token_influence}
\end{equation}
We write $\hat q_{\mathrm{grad}}$ for the corresponding empirical profile.
Appendix~\ref{app:blankout} validates the comparison using token blank-out as
an alternative, gradient-free influence measure.
\Cref{fig:rollout_comparison_main} shows
$\hat q_{\mathrm{grad}}$ for \texttt{falcon-rw-7b} and \texttt{qwen3-8b-base}. In both
cases, the empirical distributions exhibit a U-shaped profile.

\begin{figure*}[t]
  \centering
  \noindent\makebox[\textwidth][c]{\texttt{falcon-rw-7b}}\\[0.75ex]
  \begin{subfigure}[t]{0.24\textwidth}
    \centering
    \includegraphics[width=\textwidth]{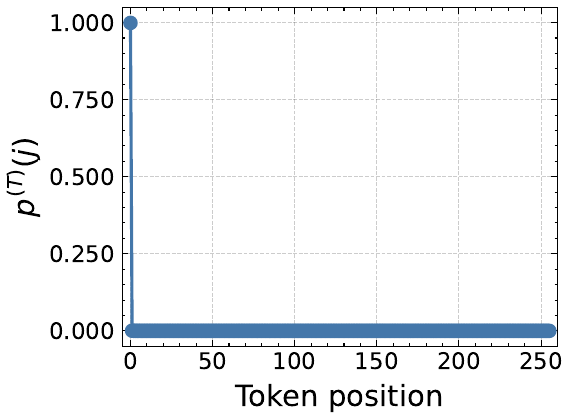}
  \end{subfigure}
  \begin{subfigure}[t]{0.24\textwidth}
    \centering
    \includegraphics[width=\textwidth]{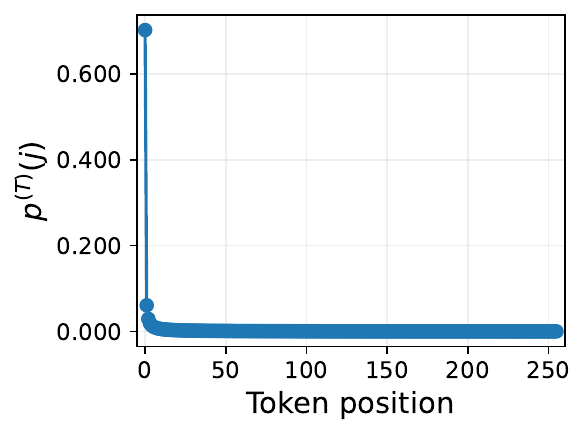}
  \end{subfigure}
  \begin{subfigure}[t]{0.24\textwidth}
    \centering
    \includegraphics[width=\textwidth]{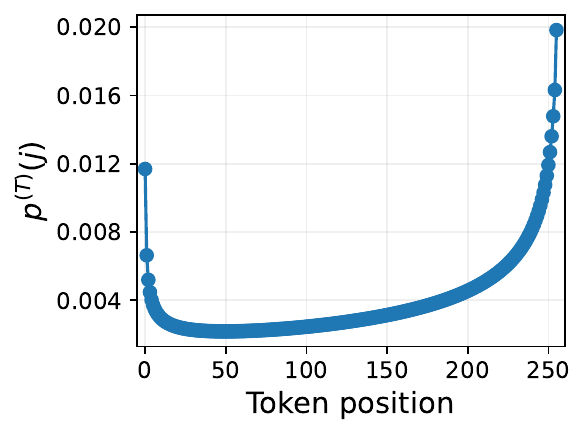}
  \end{subfigure}
  \begin{subfigure}[t]{0.24\textwidth}
    \centering
    \includegraphics[width=\textwidth]{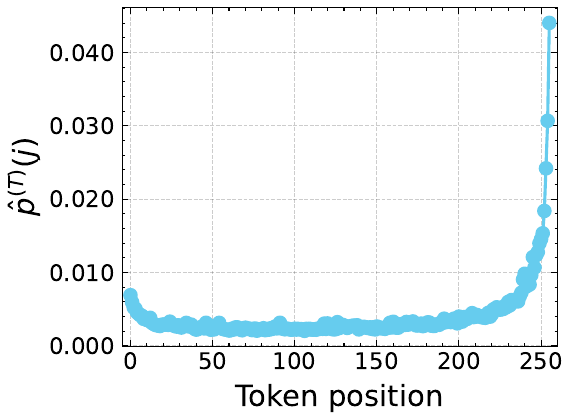}
  \end{subfigure}
  \noindent\makebox[\textwidth][c]{\texttt{qwen3-8b-base}}\\[0.75ex]
  \begin{subfigure}[t]{0.24\textwidth}
    \centering
    \includegraphics[width=\textwidth]{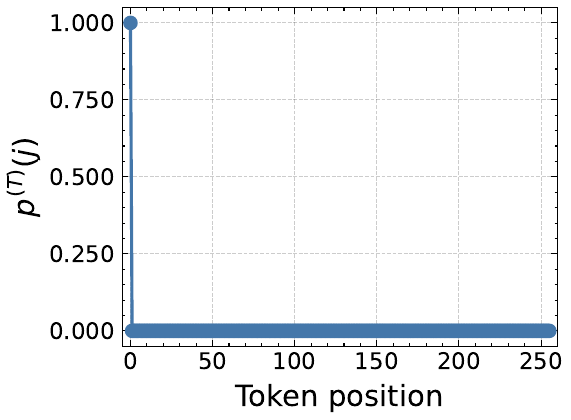}
    \caption{Attention-only rollout ($\lambda_t~=~1$)}
  \end{subfigure}
  \begin{subfigure}[t]{0.24\textwidth}
    \centering
    \includegraphics[width=\textwidth]{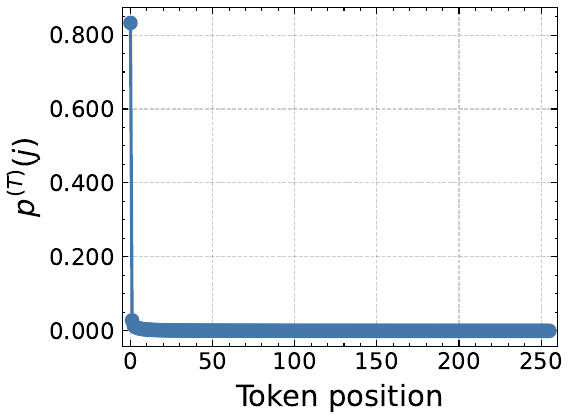}
    \caption{Fixed residual rollout ($\lambda_t~=~\tfrac12$)}
  \end{subfigure}
  \begin{subfigure}[t]{0.24\textwidth}
    \centering
    \includegraphics[width=\textwidth]{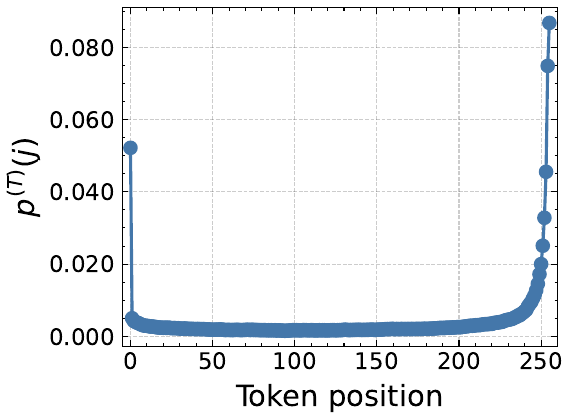}
    \caption{Adaptive residual rollout}
  \end{subfigure}
  \begin{subfigure}[t]{0.24\textwidth}
    \centering
    \includegraphics[width=\textwidth]{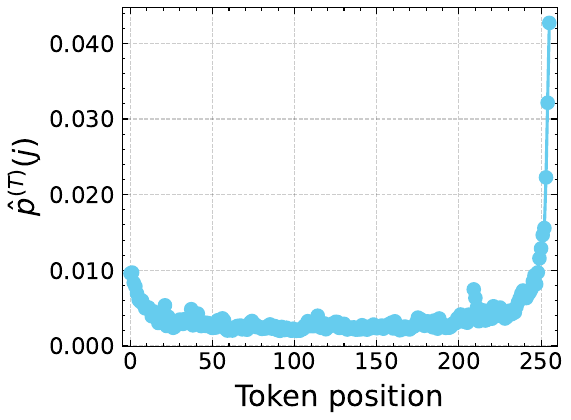}
    \caption{Measured input token influence}
  \end{subfigure}
  \caption{
Final-token profiles at sequence length $n=256$ for the 36-layer \texttt{falcon-rw-7b} (top) and the 36-layer \texttt{qwen3-8b-base} (bottom), representing ALiBi and RoPE, respectively.
Panels (a)--(c) show the Wu et al. attention-only rollout ($\lambda_t=1$),
the Abnar--Zuidema fixed residual rollout ($\lambda_t=\tfrac12$), and our
adaptive rollout with measured schedules $\{\bar\lambda_t\}$, respectively.
Panel (d) shows the gradient influence proxy $\hat q_{\mathrm{grad}}$.
}
  \label{fig:rollout_comparison_main}
\end{figure*}

\noindent\textbf{Qualitative validation.}
Consistent with the controlled analysis above, our adaptive residual-aware rollout
reproduces the dominant U-shaped structure observed in the
measured distribution. For the ALiBi model shown here, this agreement is obtained without
input-dependent content logits.
By contrast, the attention-only rollout (a) collapses
to the first token and fails to reproduce the empirical profile.

\noindent\textbf{Quantitative validation.}
\Cref{tab:wasserstein_large}
reports normalized $1$-Wasserstein distances between the rollout variants and
the measured influence distribution at sequence length $n=256$, using ground
metric $d(i,j)=|i-j|/(n-1)$. Appendix~\ref{app:rollout} reports
additional model-dataset results. Appendix~\ref{app:ablations} contains the
corpus-level approximation, context-length, and mixing-coefficient ablations,
and
Appendix~\ref{app:blankout} replaces gradient influence by token blank-out.
Across all 30 model-dataset comparisons, our adaptive rollout (c) attains a lower
Wasserstein distance than both the attention-only rollout (a) and the fixed
residual rollout (b).
As an independent intervention test, the rollout correctly predicts the
point-estimate direction of primacy and recency changes under attention-branch
rescaling in all 24 comparisons across six models
(\cref{app:mixing-intervention}).

Together, these results show that adaptive layerwise residual mixing recovers
the broad, often U-shaped positional profiles missed by both attention-only
and fixed-residual rollout.


\section{Discussion and limitations}

\noindent\textbf{Architectural scope.}
Our analysis isolates architectural position bias induced by attention,
masking, positional encodings, and residual mixing. The rollout uses only
forward-pass quantities, independently of the influence profiles used for
evaluation, and characterizes token-mixing priors rather than full
representation dynamics.
We do not model normalization as a separate force; nevertheless, the
content-free ALiBi rollout closely matches measured cumulative influence,
indicating that an explicit normalization term is unnecessary for obtaining
this predictive approximation.

\noindent\textbf{Positional encodings.}
The finite-depth comparison theorem is stated for ordered transport kernels,
with an analytic additive-logit corollary for ALiBi. The projection-aware RoPE
rollout retains corpus-averaged forward contributions rather than removing
content entirely, but does not use the measured influence target.

\noindent\textbf{Lost-in-the-Middle.}
Our analysis characterizes a U-shaped influence prior through theoretical rollout and gradient-based input-token attribution.
Lost-in-the-Middle, as introduced by \citet{liu2307}, is a downstream
task-performance phenomenon; directly connecting this prior to retrieval-style
degradation is left for future work.


\section{Conclusion}

We developed a residual-aware theory of cumulative information propagation in
which an adaptive layerwise schedule controls the balance between identity and
attention paths. At finite depth, the theory isolates the competing positional
drifts induced by causal masking, residual mixing, and positional encoding. In
pretrained ALiBi- and RoPE-based models, their composition predicts broad,
often U-shaped influence profiles. Across 30 model-dataset settings, our
rollout reduces the mean Wasserstein distance to measured token influence from
$0.62$ for attention-only rollout and $0.60$ for fixed residual mixing to
$0.09$, an approximately $85\%$ reduction relative to either baseline.

At infinite depth, cumulative mixing, rather than the mere presence of a
residual path, determines whether later-token mass is retained or collapses to
the first token under the stated boundedness conditions. Together, these
results establish adaptive residual mixing as a key component connecting
cumulative attention transport to the systematic position bias observed in
causal Transformers.

\subsection*{AI use statement}
Generative AI tools were used to assist with developing and refining
proofs, implementing code for experiments, retrieving and summarizing related
work, and drafting and editing the manuscript. The authors independently checked the mathematical
claims, proofs, references, implementation, and reported results, and take
responsibility for the final content of the paper.

\subsection*{Reproducibility statement}
The appendix provides complete proofs, kernel constructions, experimental
details, and all ablations. The complete code is available at
\url{https://github.com/ml-uhh/position-bias}. It includes
a locked software environment, experiment manifests, and commands for reproducing the
measurements, tables, and figures. Additional
protocol, uncertainty, and compute details are given in
Appendix~\ref{app:reproducibility}.

\bibliographystyle{plainnat}
\bibliography{refs}


\clearpage
\appendix

\section*{Appendix overview}

The appendix is organized as follows.\\
Appendix~\ref{app:proofs} contains the proofs of the finite- and infinite-depth results.\\
Appendix~\ref{sec:approx} gives finite-depth rollout representations and simplifications.\\
Appendix~\ref{sec:shape} gives conditional finite-mixing U-shape results.\\
Appendix~\ref{app:additional} gives further shape and approximation results.\\
Appendix~\ref{app:infinite-sanity} gives the fixed-kernel exponential sanity check.\\
Appendix~\ref{app:kernel-constructions} specifies the ALiBi and RoPE kernel constructions.\\
Appendix~\ref{app:rollout} provides additional qualitative and quantitative rollout results.\\
Appendix~\ref{app:lambda} reports the measured residual-mixing schedules.\\
Appendix~\ref{app:ablations} reports the corpus-level approximation, context-length, influence-measurement, intervention, and mixing-coefficient studies.\\
Appendix~\ref{app:reproducibility} provides protocol, uncertainty, and compute details.\\
Appendix~\ref{app:licenses} summarizes license names, links, and citations for the models and datasets used in this paper.


\section{Proofs of theoretical results}
\label{app:proofs}

This appendix contains the proofs of the finite-depth propositions from \cref{sec:finite-depth} and the infinite-depth theorem stated in \cref{sec:infinite-depth}.

\subsection{Proofs of \cref{sec:finite-depth} (finite-depth case)}
\label{app:finite-depth}

First, we need to state a few definitions and lemmas that we will use in the proofs.


\noindent\textbf{Stochastic order on positions.}
Let $\mu,\nu$ be probability distributions on $\{1,\dots,n\}$.
We write
\[
\mu \succeq_{\mathrm{FOSD}} \nu\qquad\text{(first-order stochastic dominance)}
\quad\Longleftrightarrow\quad
\sum_{j=1}^{m}\mu(j)\ \le\ \sum_{j=1}^{m}\nu(j)\qquad \forall m\in\{1,\dots,n\}.
\]
Thus $\mu \succeq_{\mathrm{FOSD}} \nu$ means that $\mu$ is shifted toward \emph{larger indices} (more recent tokens).
Equivalently, for every monotonically increasing test function $\varphi:\{1,\dots,n\}\to\mathbb{R}$,
\begin{equation}
\label{eq:fods-test}
\mu \succeq_{\mathrm{FOSD}} \nu
\quad\Longleftrightarrow\quad
\mathbb{E}_{X\sim\mu}[\varphi(X)] \ \ge\ \mathbb{E}_{X\sim\nu}[\varphi(X)].
\end{equation}
(Here $\mathbb{E}_{X\sim\mu}[\varphi(X)] = \sum_{j=1}^n \mu(j)\varphi(j)$, and similarly for $\nu$.)

\noindent\textbf{Kernels and rollout.}
A row-stochastic matrix $K\in\mathbb{R}^{n\times n}$ is viewed as a Markov kernel on $\{1,\dots,n\}$.
For a distribution $\mu$ (row vector), $\mu K$ denotes the pushforward distribution.
In our setting, the layer operator is
\[
R^{(t)} \;=\; (1-\lambda_t)I + \lambda_t A^{(t)},
\]
where $A^{(t)}$ is row-stochastic (attention) and $\lambda_t\in[0,1]$.
The rollout distribution of the final token after $T$ layers is
\[
\mu_T^\top \;=\; e_n^\top P^{(T)} \;=\; e_n^\top R^{(T)}R^{(T-1)}\cdots R^{(1)}.
\]

Throughout Appendix~\ref{app:finite-depth}, the kernels $A^{(t)}$ and mixing
coefficients $\lambda_t$ are any specified deterministic quantities satisfying
the assumptions of the corresponding result. Corpus-averaged kernels and
mixing schedules are one empirical instantiation, introduced only in
\cref{sec:methodology}.

\begin{definition}[Stochastically monotone (isotone) kernel]
\label{def:stoch-mon}
A row-stochastic kernel $K$ is \emph{stochastically monotone} if its row laws are monotone in FOSD:
\[
i<i' \quad\Longrightarrow\quad K(i,\cdot)\ \preceq_{\mathrm{FOSD}}\ K(i',\cdot).
\]
Equivalently, for every monotonically increasing $\varphi$, the function
$
 g_\varphi(i):=\sum_{j=1}^n K(i,j)\varphi(j)
$
is monotonically increasing in $i$.
\end{definition}

\begin{lemma}[Order preservation under a stochastically monotone kernel]
\label{lem:order-preserve}
Let $K$ be stochastically monotone. If $\mu \succeq_{\mathrm{FOSD}} \nu$, then
\[
\mu K \succeq_{\mathrm{FOSD}} \nu K.
\]
\end{lemma}

\begin{proof}
Fix any monotonically increasing $\varphi$. Let $g(i):=\sum_j K(i,j)\varphi(j)$.
By Definition~\ref{def:stoch-mon}, $g$ is monotonically increasing.
Using \cref{eq:fods-test},
\[
\mathbb{E}_{X\sim\mu K}[\varphi(X)] = \sum_i \mu(i) g(i) \ \ge\ \sum_i \nu(i) g(i) = \mathbb{E}_{X\sim\nu K}[\varphi(X)].
\]
Since this holds for all monotonically increasing $\varphi$, $\mu K \succeq_{\mathrm{FOSD}} \nu K$ by \cref{eq:fods-test}.
\end{proof}

\begin{lemma}[Residual kernels inherit stochastic monotonicity]
\label{lem:residual-preserve}
If $A$ is stochastically monotone, then for any $\lambda\in[0,1]$ the residual kernel
$R_\lambda=(1-\lambda)I+\lambda A$ is also stochastically monotone.
\end{lemma}

\begin{proof}
Fix a monotonically increasing $\varphi$ and define $g_A(i):=\sum_j A(i,j)\varphi(j)$.
Since $A$ is stochastically monotone, $g_A$ is monotonically increasing.
For $R_\lambda$,
\[
g_{R_\lambda}(i):=\sum_j R_\lambda(i,j)\varphi(j)=(1-\lambda)\varphi(i)+\lambda g_A(i),
\]
which is monotonically increasing as a convex combination of monotonically increasing functions. Hence $R_\lambda$ is stochastically monotone.
\end{proof}

\begin{lemma}[Rowwise dominance implies dominance after one step]
\label{lem:rowwise-dominance}
Let $K,K'$ be row-stochastic kernels. Assume row-wise dominance, i.e., for all $i$,
\[
K(i,\cdot)\ \succeq_{\mathrm{FOSD}}\ K'(i,\cdot).
\]
Then for any distribution $\mu$,
\[
\mu K \succeq_{\mathrm{FOSD}} \mu K'.
\]
\end{lemma}

\begin{proof}
Fix a monotonically increasing $\varphi$ and define $g(i):=\sum_j K(i,j)\varphi(j)$ and $g'(i):=\sum_j K'(i,j)\varphi(j)$.
By the row-wise dominance assumption and \cref{eq:fods-test}, $g(i)\ge g'(i)$ for all $i$.
Therefore,
\[
\mathbb{E}_{X\sim\mu K}[\varphi(X)]=\sum_i \mu(i)g(i)\ \ge\ \sum_i\mu(i)g'(i)=\mathbb{E}_{X\sim\mu K'}[\varphi(X)].
\]
\end{proof}

\begin{lemma}[Monotone-vector propagation]
\label{lem:monotone-vector}
If $K$ is stochastically monotone and $v\in\mathbb{R}^n$ is monotonically decreasing in its index, then $Kv$ is monotonically decreasing.
\end{lemma}
\begin{proof}
For $i<i'$, $K(i,\cdot)\preceq_{\mathrm{FOSD}} K(i',\cdot)$.
Taking $\varphi=-v$ (which is monotonically increasing since $v$ is monotonically decreasing) and using \cref{eq:fods-test} yields
$\sum_j K(i,j)(-v_j)\le \sum_j K(i',j)(-v_j)$, i.e. $(Kv)_i\ge (Kv)_{i'}$.
\end{proof}


\subsubsection{Primacy drift under causal masking}

\begin{proof}[Proof of Proposition~\ref{prop:causal-primacy}]
Fix $k<n$ and define $\mathbf{1}_k\in\mathbb{R}^n$ by $(\mathbf{1}_k)_j=\mathbbm{1}_{\{j\le k\}}$.
Let $m^{(t)} := P^{(t)}\mathbf{1}_k$, so that the final-token prefix mass equals
\[
\sum_{j\le k}\mu_t(j)\;=\;e_n^\top P^{(t)}\mathbf{1}_k \;=\; m_n^{(t)}.
\]
Since each $A^{(t)}$ is stochastically monotone by Assumption~\ref{ass:stoch-monotone-A}, Lemma~\ref{lem:residual-preserve} implies each $R^{(t)}$ is also stochastically monotone.

Since $\mathbf{1}_k$ is monotonically decreasing and $m^{(t+1)}=R^{(t+1)}m^{(t)}$, Lemma~\ref{lem:monotone-vector} implies
$m^{(t)}$ is monotonically decreasing for every $t$.
Hence, $m^{(t)}_i\ge m^{(t)}_n$ for all $i$.
Using row-stochasticity,
\[
m^{(t+1)}_n = \sum_{i=1}^n R^{(t+1)}_{n i}\, m^{(t)}_i \ \ge\ \sum_{i=1}^n R^{(t+1)}_{n i}\, m^{(t)}_n = m^{(t)}_n,
\]
which proves that the final-token prefix mass $\sum_{j\le k}\mu_t(j)$ is monotonically increasing in depth.
\end{proof}

\subsubsection{Stronger residual (smaller $\lambda$) induces recency drift}
\begin{proof}[Proof of Proposition~\ref{prop:residual-recency}]
Fix a layer and write $R_\lambda=(1-\lambda)I+\lambda A$, where $A$ is row-stochastic and causal,
so that $\mathrm{supp}(A(i,\cdot))\subseteq\{1,\dots,i\}$ for each $i$.
Under our convention that larger indices are more recent, for any distribution $\nu$ supported on $\{1,\dots,i\}$ we have
$\delta_i \succeq_{\mathrm{FOSD}} \nu$, where $\delta_i$ is the Kronecker delta.
Therefore, for any $0\le \lambda \le \lambda' \le 1$ and each $i$,
\begin{equation}
\label{eq:rowwise-lambda-correct}
R_\lambda(i,\cdot)=(1-\lambda)\delta_i+\lambda A(i,\cdot)
\ \succeq_{\mathrm{FOSD}}\
(1-\lambda')\delta_i+\lambda' A(i,\cdot)=R_{\lambda'}(i,\cdot).
\end{equation}

Now compare two schedules $\{\lambda_t\}$ and $\{\lambda'_t\}$ with
$\lambda_t\le \lambda'_t$ for all $t$. To respect the product order
$P^{(T)}=R^{(T)}\cdots R^{(1)}$, set
$\mu^{(T+1)}={\mu'}^{(T+1)}=e_n^\top$ and, for
$t=T,T-1,\ldots,1$, define
\[
\mu^{(t)}=\mu^{(t+1)}R^{(t)}_{\lambda_t},\qquad
{\mu'}^{(t)}={\mu'}^{(t+1)}R^{(t)}_{\lambda'_t}.
\]
Suppose $\mu^{(t+1)}\succeq_{\mathrm{FOSD}}{\mu'}^{(t+1)}$.
Rowwise dominance from \cref{eq:rowwise-lambda-correct}, followed by order
preservation under the stochastically monotone kernel
$R^{(t)}_{\lambda'_t}$, gives
\[
\mu^{(t+1)}R^{(t)}_{\lambda_t}
\succeq_{\mathrm{FOSD}}
\mu^{(t+1)}R^{(t)}_{\lambda'_t}
\succeq_{\mathrm{FOSD}}
{\mu'}^{(t+1)}R^{(t)}_{\lambda'_t}.
\]
Backward induction therefore yields
$e_n^\top P^{(T)}=\mu^{(1)}\succeq_{\mathrm{FOSD}}
{\mu'}^{(1)}=e_n^\top {P'}^{(T)}$.
\end{proof}

\subsubsection{Recency-favoring positional encodings induce recency drift}

\begin{proof}[Proof of Proposition~\ref{prop:pe-vs-nope}]
Fix a layer and row, and let
\[
\pi_\theta(j)\propto\exp(s_{ij}+\theta b_{ij}),
\qquad 0\leq\theta\leq1,
\]
on the row's causal support. For a prefix indicator
$h_m(j)=\mathbbm{1}_{\{j\leq m\}}$, the softmax derivative identity yields
\[
\frac{\mathrm d}{\mathrm d\theta}\mathbb E_{\pi_\theta}[h_m]
=\operatorname{Cov}_{\pi_\theta}(h_m,b_{i\cdot})\leq0.
\]
The inequality holds because $h_m$ is nonincreasing in $j$, whereas
$b_{ij}$ is nondecreasing. Hence every prefix mass decreases with $\theta$,
which proves
$A_{\mathrm{PE}}^{(t)}(i,\cdot)\succeq_{\mathrm{FOSD}}
A_0^{(t)}(i,\cdot)$. Mixing both kernels with the same identity row gives
$R_{\mathrm{PE}}^{(t)}(i,\cdot)\succeq_{\mathrm{FOSD}}
R_0^{(t)}(i,\cdot)$.
By Assumption~\ref{ass:stoch-monotone-A} and
Lemma~\ref{lem:residual-preserve}, every $R_0^{(t)}$ is stochastically
monotone. Set $\mu_{\mathrm{PE}}^{(T+1)}=\mu_0^{(T+1)}=e_n^\top$ and define
$\mu_a^{(t)}=\mu_a^{(t+1)}R_a^{(t)}$ for $t=T,T-1,\ldots,1$.
Rowwise dominance and order preservation then give inductively
\[
\mu_{\mathrm{PE}}^{(t+1)}R_{\mathrm{PE}}^{(t)}
\succeq_{\mathrm{FOSD}}
\mu_{\mathrm{PE}}^{(t+1)}R_0^{(t)}
\succeq_{\mathrm{FOSD}}
\mu_0^{(t+1)}R_0^{(t)}.
\]
Hence $p_{\mathrm{PE}}^{(T)}\succeq_{\mathrm{FOSD}}p_0^{(T)}$.
For ALiBi, $b_{ij}=-m(i-j)$ with $m\geq0$, which is nondecreasing in $j$.
\end{proof}

\subsection{Proof of \cref{thm:residual_dichotomy} (infinite-depth case)}
\label{app:proof_residual_dichotomy}

We recall the residual-aware transition $R^{(t)} = (1-\lambda_t)I + \lambda_t A^{(t)}$ and the rollout $P^{(T)} = R^{(T)}R^{(T-1)}\cdots R^{(1)}$.

\begin{proof}[Proof of Lemma~\ref{lem:bound_epsilon}]
To prove (i), let $X^{(t)}$ denote the states of the normalized analytical
dynamics corresponding to the transition $R^{(t)}$,
$X^{(t)} = (1-\lambda_t)X^{(t-1)} + \lambda_t A^{(t)} X^{(t-1)} W_V^{(t)}$,
which tracks the relative identity and attention contributions.
Using the triangle inequality and submultiplicativity of the operator norm, and noting that the row-stochastic attention matrix $A^{(t)}$ does not increase the maximum row norm (i.e., its induced $\infty$-norm on the rows is 1):
\[
\max_i \|X_{i,:}^{(t)}\|_2 \le (1-\lambda_t) \max_i \|X_{i,:}^{(t-1)}\|_2 + \lambda_t \max_i \|X_{i,:}^{(t-1)}\|_2 \|W_V^{(t)}\|_2.
\]
Letting $M^{(t)} = \max_i \|X_{i,:}^{(t)}\|_2$ and $C_t = \|W_V^{(t)}\|_2$, we have:
\[
M^{(t)} \le M^{(t-1)} \bigl((1-\lambda_t) + \lambda_t C_t\bigr) = M^{(t-1)} \bigl(1 + \lambda_t (C_t - 1)\bigr) \le M^{(t-1)} \bigl(1 + \lambda_t \max(0, C_t - 1)\bigr).
\]
Unrolling this recursion through $T$ gives:
\[
M^{(T)} \le M^{(0)} \prod_{s=1}^{T} \bigl(1 + \lambda_s \max(0, C_s - 1)\bigr) \le M^{(0)} \exp\left( \sum_{s=1}^{T} \lambda_s \max(0, C_s - 1) \right).
\]
By Assumption~\ref{ass:bounded_v}, the sum in the exponent converges to some finite value. Let $K = \exp\bigl(\sum_{s=1}^{\infty} \lambda_s \max(0, C_s - 1)\bigr) < \infty$. Since the initial input is bounded by $C$, we have $M^{(T)} \le C \cdot K$. Setting $C_X = C \cdot K$ establishes the uniform bound on the trajectories.

For (ii), we utilize the result of part (i). The attention logits are
$\ell^{(t)}_{ij} = s^{(t)}_{ij} + b^{(t)}_{ij}$, where
\[
s^{(t)}_{ij}
= (X_{i,:}^{(t-1)} W_Q^{(t)})(X_{j,:}^{(t-1)} W_K^{(t)})^\top.
\]
By Cauchy-Schwarz and Assumption~\ref{ass:bounded_qk}, $|s^{(t)}_{ij}| \le (C_X C)^2$.
Thus, for any admissible query $i$ and keys $j, k$, the content logit range is uniformly bounded: $|s^{(t)}_{ij} - s^{(t)}_{ik}| \le 2(C_X C)^2$.
Combined with Assumption~\ref{ass:bounded_b}, the full admissible logit range is bounded:
\[
|\ell^{(t)}_{ij} - \ell^{(t)}_{ik}| \le |s^{(t)}_{ij} - s^{(t)}_{ik}| + |b^{(t)}_{ij} - b^{(t)}_{ik}| \le 2(C_X C)^2 + B =: M.
\]
Since softmax is invariant to row-wise additive constants, we can lower bound any admissible entry $A_{ij}^{(t)}$ by writing it as $( \sum_{k} \exp(\ell_{ik}^{(t)} - \ell_{ij}^{(t)}) )^{-1}$.
Because $\ell_{ik}^{(t)} - \ell_{ij}^{(t)} \le M$, we obtain $A_{ij}^{(t)} \ge \frac{1}{n e^M} =: \varepsilon > 0$, which is strictly positive for any fixed sequence length $n$.
\end{proof}

\begin{proof}[Proof of \cref{thm:residual_dichotomy}(i)]
Assume $\sum_{t=1}^{\infty} \lambda_t < \infty$. We first lower bound the diagonal entries of $R^{(t)}$ using Lemma~\ref{lem:bound_epsilon}:
\[
R^{(t)}_{ii} = (1-\lambda_t) + \lambda_t A^{(t)}_{ii} \ge 1 - \lambda_t + \lambda_t \varepsilon = 1 - (1-\varepsilon)\lambda_t.
\]
Since each $A^{(t)}$ (and thus $R^{(t)}$) is lower-triangular due to causal masking, the diagonal entries of their product are the products of their diagonal entries:
\[
P^{(T)}_{ii} = (R^{(T)} \cdots R^{(1)})_{ii} = \prod_{t=1}^{T} R^{(t)}_{ii} \ge \prod_{t=1}^{T} \bigl(1-(1-\varepsilon)\lambda_t\bigr).
\]
Since $\sum \lambda_t < \infty$, the infinite product $\prod_{t=1}^{\infty} (1 - (1-\varepsilon)\lambda_t)$ converges to a strictly positive constant. Thus, $\liminf_{T \to \infty} P^{(T)}_{ii} > 0$, and the attention mass cannot collapse entirely to the first token.
\end{proof}

\begin{proof}[Proof of \cref{thm:residual_dichotomy}(ii)]
Assume $\sum_{t=1}^{\infty} \lambda_t = \infty$. For any token $i$, we track the total weight $W_{i,>1}^{(T)}$ assigned to all positions excluding the first token. Since each attention matrix $A^{(t)}$ must assign at least $\varepsilon$ mass to the first token, it reduces the remaining weight by a factor of at least $(1-\varepsilon)$.

The residual $R^{(t)} = (1-\lambda_t)I + \lambda_t A^{(t)}$ acts as a weighted average between maintaining the current weight distribution and applying this reduction. Consequently, the weight assigned to positions $j > 1$ shrinks at each layer by a factor of $(1 - \varepsilon \lambda_t)$. Concretely, using $A_{jj}^{(t)} \le 1 - (j-1)\varepsilon$ from Lemma~\ref{lem:bound_epsilon}, the diagonal entries of the single-layer transition satisfy:
\[
R_{jj}^{(t)} = (1-\lambda_t) + \lambda_t A_{jj}^{(t)} \le 1 - \lambda_t + \lambda_t(1 - (j-1)\varepsilon) = 1 - (j-1)\varepsilon \lambda_t.
\]
Compounding these contractions through the rollout $P^{(T)} = R^{(T)} \cdots R^{(1)}$, and noting that these are lower-triangular matrices so the diagonal entries of the product are the products of the diagonal entries, we obtain:
\[
P_{jj}^{(T)} = \prod_{s=1}^T R_{jj}^{(s)} \le \prod_{s=1}^T \bigl(1 - (j-1)\varepsilon \lambda_s\bigr) \le \exp\left( -(j-1)\varepsilon \sum_{s=1}^T \lambda_s \right).
\]
For $j > 1$, the divergence of $\sum \lambda_s$ implies $P_{jj}^{(T)} \to 0$ as $T \to \infty$. 

To establish the bound for all $j \le i \le n$, we observe that for any $j \in \{2, \dots, n\}$, the lower-right submatrix spanning indices $\{j, \dots, n\}$ is closed under multiplication. Let $R^{(s)}_{\ge j}$ denote this submatrix. Since $R^{(s)}$ is row-stochastic, for any row $i \ge j$, the sum of its entries in the submatrix is
\[
\sum_{k=j}^i R^{(s)}_{ik} = 1 - \sum_{k=1}^{j-1} R^{(s)}_{ik} = 1 - \lambda_s \sum_{k=1}^{j-1} A^{(s)}_{ik}.
\]
Using Lemma~\ref{lem:bound_epsilon}(ii), we have $A^{(s)}_{ik} \ge \varepsilon$, so $\sum_{k=1}^{j-1} A^{(s)}_{ik} \ge (j-1)\varepsilon$. Thus,
\[
\|R^{(s)}_{\ge j}\|_\infty = \max_{i \ge j} \sum_{k=j}^i R^{(s)}_{ik} \le 1 - (j-1)\varepsilon \lambda_s.
\]
By submultiplicativity of the induced $\infty$-norm, the rollout submatrix $P^{(T)}_{\ge j} = R^{(T)}_{\ge j} \cdots R^{(1)}_{\ge j}$ satisfies
\[
\|P^{(T)}_{\ge j}\|_\infty \le \prod_{s=1}^T \|R^{(s)}_{\ge j}\|_\infty \le \prod_{s=1}^T \bigl(1 - (j-1)\varepsilon \lambda_s\bigr) \le \exp\left( -(j-1)\varepsilon \sum_{s=1}^T \lambda_s \right).
\]
Since $P^{(T)}_{ij}$ for $i \ge j$ is an entry in this submatrix, it is bounded by the same value, which proves the exponential decay for all $j > 1$.

As $T \to \infty$, the divergence of $\sum \lambda_t$ forces this upper bound to zero. This implies that the attention mass for all tokens $j > 1$ vanishes, leaving the first token to absorb the entire unit of row mass: $\lim_{T \to \infty} P_{i1}^{(T)} = 1$.
\end{proof}

\section{Finite-Depth Representations and Simplifications}
\label{sec:approx}
We now make input dependence explicit: $A^{(t)}(x)$ and $\lambda_t(x)$ are the kernel and mixing coefficient for input $x$, and $R^{(t)}(x)=(1-\lambda_t(x))I+\lambda_t(x)A^{(t)}(x)$. No particular calibration formula is needed.

\paragraph{Corpus averages.}
Write $\bar\lambda_t=\E_{x\sim\mathcal D}\lambda_t(x)$ and
$\bar A^{(t)}=\E_{x\sim\mathcal D} A^{(t)}(x)$.
Separate averaging defines $R^{(t)}_{\mathrm{mf}}=(1-\bar\lambda_t)I+
\bar\lambda_t\bar A^{(t)}$. The joint mean instead satisfies
\begin{equation}
 \E_{x\sim\mathcal D} R^{(t)}(x)=(1-\bar\lambda_t)I+\bar\lambda_t A^{(t)}_{\mathrm{eff}},
 \qquad A^{(t)}_{\mathrm{eff}}=
 \frac{\E_{x\sim\mathcal D}[\lambda_t(x)A^{(t)}(x)]}{\bar\lambda_t},
 \label{eq:effective}
\end{equation}
when $\bar\lambda_t>0$. If $\bar\lambda_t=0$, the mean operator is $I$
and $A^{(t)}_{\mathrm{eff}}$ may be chosen arbitrarily. Both constructions are
stochastic kernels. They coincide only under an appropriate factorization
of the within-layer expectation; neither automatically commutes with
expectation over the full product.

\paragraph{Two mean-field steps.}
For input-dependent operators, the approximations
\begin{align}
 \E_{x\sim\mathcal D}[R^{(T)}(x)\cdots R^{(1)}(x)]
 &\approx (\E_{x\sim\mathcal D}R^{(T)}(x))\cdots
 (\E_{x\sim\mathcal D}R^{(1)}(x)), \label{eq:cross-layer}\\
 \E_{x\sim\mathcal D}[\lambda_t(x)A^{(t)}(x)]
 &\approx (\E_{x\sim\mathcal D}\lambda_t(x))
 (\E_{x\sim\mathcal D}A^{(t)}(x)) \label{eq:within-layer}
\end{align}
discard different dependencies. \Cref{eq:cross-layer} concerns cross-layer
composition; \cref{eq:within-layer} concerns the two factors within one
layer. Neither is an identity in general. In the rest of this section,
$A^{(t)},\lambda_t$ denote a specified deterministic construction, using
either separate means or the exact joint-mean kernel in \cref{eq:effective}.

\begin{proposition}[Ordered attention-step representation]
\label{prop:bernoulli}
Let $Z_t$ be independent Bernoulli variables with parameters $\lambda_t$.
With $\widehat R^{(t)}=I$ when $Z_t=0$ and $\widehat R^{(t)}=A^{(t)}$ when $Z_t=1$,
\[
 P^{(T)}=\E_Z[\widehat R^{(T)}\cdots\widehat R^{(1)}].
\]
If $A^{(t)}=A$ for all layers, this becomes
\begin{equation}
 P^{(T)}=\E[A^S]=\sum_{k=0}^T w_k A^k,
 \quad S=\sum_t Z_t,\quad
 \sum_{k=0}^T w_k z^k=\prod_{t=1}^T(1-\lambda_t+\lambda_t z).
 \label{eq:count}
\end{equation}
\end{proposition}
\begin{proof}
Independence permits entrywise expectation of the ordered product to factor as
$(\E\widehat R^{(T)})\cdots(\E\widehat R^{(1)})
=R^{(T)}\cdots R^{(1)}$. When every attention kernel equals $A$, each
realization reduces to $A^S$. The generating function of the independent
Bernoulli sum gives the weights in \cref{eq:count}.
\end{proof}
Identity steps disappear from each sampled path. For different kernels,
the number of attention steps alone does not determine their ordered product.

\paragraph{A common kernel.}
For $\tau=\sum_t\lambda_t>0$, define
\begin{equation}
 A_*=\frac{\sum_t\lambda_t A^{(t)}}{\tau},\qquad
 P_{\rm common}=[(1-\lambda_T)I+\lambda_T A_*]\cdots
 [(1-\lambda_1)I+\lambda_1 A_*].
 \label{eq:common}
\end{equation}
Then $A_*$ is stochastic and
$\sum_t\lambda_t(A^{(t)}-I)=\tau(A_*-I)$. Thus the integrated generator is
preserved, but generally $P_{\rm common}\ne P^{(T)}$. For $\tau=0$ both
rollouts equal $I$ and the common kernel is immaterial.
Concurrently, \citet{chowdhury2026middle} derive closed forms for the
uniform causal (Ces\`aro) kernel
$M_{ij}=i^{-1}\mathbf{1}\{j\le i\}$ with constant residual mixing $\alpha$.
This kernel assigns equal logits to all admissible positions and therefore
contains neither content dependence nor an ALiBi or RoPE contribution. Their
rollout $[(1-\alpha)I+\alpha M]^T$ is exactly the constant-kernel,
constant-schedule special case of \cref{eq:common}; at $\alpha=1/2$, it is
also the Abnar--Zuidema fixed-residual rule specialized to the Ces\`aro
kernel. Their sequence-length continuum yields a decreasing primacy density
and an atom at the final position, rather than the broad finite U-shaped
profiles studied here. It is also distinct from the exponential construction
below, which recovers this rollout only under $\alpha=\tau/T$ as
$T\to\infty$. At fixed sequence length and any constant $\alpha>0$, the
residual atom has mass $(1-\alpha)^T\to0$ and the infinite-depth rollout
collapses to the first token.
Related continuous-depth work by \citet{huan2026adjoint} studies positional
adjoint sensitivity in a causal residual flow. Its controlled simulations use
a fixed Ces\`aro kernel and the propagator
$[I+(T\alpha/M)(C\otimes S)]^M$. In the scalar-feature specialization,
normalization removes a common growth factor and the positional transport
converges to $\exp[T\alpha(C-I)]$, the fixed-kernel exponential model in
\cref{eq:exponential}. Chowdhury expresses this through a discrete
fixed-residual rollout. Huan and Yuan express it through a finite-time
continuous-depth flow and a right-biased terminal adjoint. After scalarization
and normalization, both belong to our common-kernel exponential picture.
Increasing $M$ refines a finite-total-mixing flow rather than introducing a
measured or adaptive layerwise mixing schedule. Their experiments are small
controlled simulations with 48 positions and four-dimensional features, and
do not compare predicted profiles with measured influence in pretrained
Transformers.

\begin{proposition}[Exponential simplification]
\label{prop:exponential}
For a stochastic kernel $A$ and $\tau\ge0$,
\begin{equation}
 P_{\exp}(\tau)=\exp[\tau(A-I)]
 =\E[A^{N_\tau}],\qquad N_\tau\sim\operatorname{Poisson}(\tau).
 \label{eq:exponential}
\end{equation}
Moreover, $[I+\tau(A-I)/M]^M\to P_{\exp}(\tau)$ as $M\to\infty$.
For integers $M\ge\tau$ each factor is stochastic.
\end{proposition}
\begin{proof}
Since $A$ commutes with $I$,
\[
e^{\tau(A-I)}=e^{-\tau}\sum_{k=0}^{\infty}\frac{\tau^k}{k!}A^k,
\]
which is the stated Poisson mixture. The Euler limit is the standard matrix
exponential limit. For $M\geq\tau$, each factor can be written as
$(1-\tau/M)I+(\tau/M)A$ and is therefore stochastic.
\end{proof}
Applied to $A_*$, this replaces the finite Bernoulli count by a Poisson
count at the same total mixing. It is a finite-$\tau$ approximation, not
the infinite-total-mixing limit. The associated row-vector ODE is
\[
 \frac{d}{d\tau}p(\tau)=p(\tau)(A-I),\qquad p(0)=e_n^\top.
\]
The exponential form solves this ODE exactly. It does not assert that the
original layer-varying rollout is an exact solution of a constant-generator ODE.

\paragraph{Computing only the final profile.}
Set $v^{(0)}=e_n^\top$ and $v^{(k+1)}=v^{(k)}A$. Then
\cref{eq:count} needs only weighted row vectors, not full matrix powers.
For \cref{eq:exponential}, truncation after $K$ powers leaves a nonnegative
remainder of total mass $\Prob(N_\tau>K)$, giving an explicit $\ell_1$
truncation error. This requires $K$ matrix-vector products and storage of
the kernel and a constant number of row vectors.

\section{Conditional Finite-Mixing U-Shapes}
\label{sec:shape}
Let $n\ge3$, let $A$ be any causal stochastic kernel, and put
$p(\tau)=e_n^\top\exp[\tau(A-I)]$. Two notions of shape will be used.
\emph{Endpoint elevation} means that $p_1$ and $p_n$ both exceed the mean
over positions $2,\ldots,n-1$. For a partition into a left prefix $L$, a
nonempty middle interval $M$, and a right suffix $D$, \emph{block U-shape}
means
\[
 \min\{p(L)/|L|,p(D)/|D|\}>p(M)/|M|.
\]
Neither definition asserts a unique minimum or excludes local reversals.

\begin{theorem}[Endpoint elevation from direct primacy and survival]
\label{thm:endpoints}
Let $b=1-A_{nn}>0$ and $\eta=A_{n1}/b$. If
\[
 \frac1{n-1}<\eta<1,\qquad
 0<\tau<\frac1b\log\frac{n-1-\eta}{1-\eta},
\]
then $p(\tau)$ has endpoint elevation. If $\eta=1$, the conclusion holds
for every finite $\tau>0$.
\end{theorem}
\begin{proof}
The generator $A-I$ defines a continuous-time Markov chain. Its exit rate
from state $n$ is $b$, and causality prevents return after exit, so
$p_n(\tau)=q:=e^{-b\tau}$. Conditional on the first exit, the chain moves
directly to state $1$ with probability $\eta$. Since state $1$ is absorbing,
$p_1(\tau)\geq\eta(1-q)$, while the total mass on positions
$2,\ldots,n-1$ is at most $(1-\eta)(1-q)$. The left endpoint exceeds the
interior average when $\eta>1/(n-1)$. The right endpoint exceeds it when
$q>(1-\eta)/(n-1-\eta)$, which is equivalent to the stated upper bound on
$\tau$. If $\eta=1$, the interior mass is zero and both endpoints have
positive mass for every finite $\tau>0$.
\end{proof}
Only two entries of $A$ enter this sufficient condition. It does not require
stochastic monotonicity. It can fail for recency-weighted rows even when
indirect paths produce primacy, which motivates the next certificate.

\begin{theorem}[Finite path certificate for block U-shape]
\label{thm:block}
Let $L=\{1,\ldots,\ell\}$, $D=\{n-r+1,\ldots,n\}$, and
$m=n-\ell-r>0$, with $\ell,r\ge1$. For $k\ge0$, set
\[
 v^{(k)}=e_n^\top A^k,\quad u_k=v^{(k)}(L),\quad s_k=v^{(k)}(D),
 \quad w_k=e^{-\tau}\frac{\tau^k}{k!}.
\]
For any integer $K\ge0$, define $\epsilon_K=\Prob(N_\tau>K)$ and
\begin{align*}
 L_K&=\sum_{k=0}^K w_k u_k+\epsilon_K u_K,\\
 D_K&=\max\left\{e^{-(1-A_{nn})\tau},\sum_{k=0}^K w_k s_k\right\}.
\end{align*}
If
\begin{equation}
 \min\left\{\frac{L_K}{\ell},\frac{D_K}{r}\right\}
 >\frac{1-L_K-D_K}{m},
 \label{eq:block-certificate}
\end{equation}
then $p(\tau)$ has block U-shape for this partition. Conversely, whenever
the profile has strict block U-shape for this fixed partition, some finite
$K$ satisfies \cref{eq:block-certificate}.
\end{theorem}
\begin{proof}
The discrete chain with kernel $A$ cannot leave the causal prefix $L$, so
$u_k$ is nondecreasing. In the Poisson mixture, the tail contribution to $L$
is at least $\epsilon_Ku_K$, hence $p(L)\geq L_K$. Both the partial Poisson
sum and the probability of never leaving state $n$ lower-bound $p(D)$, so
$p(D)\geq D_K$. It follows that $p(M)\leq1-L_K-D_K$, and
\cref{eq:block-certificate} proves sufficiency. As $K\to\infty$, $L_K$ and
$D_K$ converge to $p(L)$ and $p(D)$, respectively. Therefore any strict block
U-shape for the fixed partition satisfies the certificate at some finite $K$.
\end{proof}
This is a computable certificate for the exponential model, not a claim
that every causal kernel has a U-shaped rollout. It depends only on the
kernel and total mixing, not an output-influence measurement. The same
path calculation applies to the exact common-kernel finite product by
using its finite weights $w_k$ from \cref{eq:count}; taking all $T$ powers
gives its exact block masses.

\paragraph{Limits of shape transfer.}
For normalized distributions, define
\[
 W_1(p,q)=\frac1{n-1}\sum_{k=1}^{n-1}|\F_p(k)-\F_q(k)|.
\]
Small $W_1$ alone does not preserve pointwise shape inequalities. Bounds
on coordinate or block-mass errors, together with positive shape margins,
are needed to transfer a certificate to the original finite rollout.
An operator perturbation bound is given in \cref{prop:perturbation}.

\section{Additional Shape and Approximation Results}
\label{app:additional}
\subsection{A structural single-valley special case}
\begin{proposition}[Decreasing rows and a terminal recency spike]
\label{prop:strict}
Suppose every row of the causal stochastic kernel $A$ is nonincreasing
in key position on its support. Let $p(\tau)=e_n^\top e^{\tau(A-I)}$,
$a=A_{n,n-1}$, $d=A_{n-1,n-1}-A_{nn}$, and
\[
 \phi_d(\tau)=\begin{cases}(e^{d\tau}-1)/d,&d\ne0,\\\tau,&d=0.
 \end{cases}
\]
Then $p_1\ge\cdots\ge p_{n-1}$ and
$p_{n-1}/p_n=a\phi_d(\tau)$. Hence $a\phi_d(\tau)<1$ gives a
decreasing prefix followed by a strict terminal rise. If all causal
entries are positive and $\tau>0$, the prefix inequalities are strict.
\end{proposition}
\begin{proof}
Pad every row by zeros to length $n$. Each padded row is decreasing, and
so is every probability-weighted mixture of rows. Thus
$v^{(k)}=e_n^\top A^k$ is decreasing for $k\ge1$. The $k=0$ term in
the Poisson mixture affects only the last coordinate. If all causal
entries are positive, for $j<n$ the contribution of row $j$ to
$v^{(2)}_j-v^{(2)}_{j+1}$ is $A_{nj}A_{jj}>0$, and all other
contributions are nonnegative. The positive Poisson weight for $k=2$
makes the prefix inequalities strict. Finally the last two ODEs are
\[
 p_n'=-(1-A_{nn})p_n,\quad
 p_{n-1}'=a p_n-(1-A_{n-1,n-1})p_{n-1}.
\]
With initial values $1$ and $0$, solving gives the stated ratio.
\end{proof}

For uniform causal attention, $A_{ij}=1/i$ when $j\le i$, this gives
\[
 0<\tau<n(n-1)\log\frac n{n-1}
 \quad\Longrightarrow\quad p_1>\cdots>p_{n-1}<p_n.
\]
This is a special-case shape with a one-token right arm. The decreasing-row
hypothesis differs from \cref{ass:stoch-monotone-A} and is not asserted for
general ALiBi or RoPE kernels.

\subsection{Why drift assumptions alone do not imply U-shape}
The causal, positive-on-support, stochastically monotone kernel
\[
 A=\begin{pmatrix}1&0&0\\0.01&0.99&0\\0.01&0.98&0.01\end{pmatrix}
\]
gives at $\tau=2$
\[
 p(2)=(1-e^{-0.02},\ e^{-0.02}-e^{-1.98},\ e^{-1.98}).
\]
The middle coordinate exceeds both endpoints. Thus stochastic
monotonicity, causal support, and positivity alone do not imply a
finite-mixing U-shape.

\subsection{Perturbation and shape margins}
\begin{proposition}[Common-kernel perturbation bound]
\label{prop:perturbation}
For stochastic kernels $A,B$ and $\tau\ge0$,
\[
 \norm{e_n^\top e^{\tau(A-I)}-e_n^\top e^{\tau(B-I)}}_1
 \le\tau\max_i\sum_j|A_{ij}-B_{ij}|.
\]
\end{proposition}
\begin{proof}
Duhamel's identity expresses the exponential difference as
\[
 \int_0^\tau e^{s(A-I)}(A-B)e^{(\tau-s)(B-I)}\,ds.
\]
The leftmost row distribution is stochastic. Its product with $A-B$
has $\ell_1$ norm at most the maximum absolute row sum of $A-B$.
Right multiplication by a stochastic matrix contracts the $\ell_1$
norm of any signed row vector. Integration proves the bound.
\end{proof}
If $\norm{p-q}_1\le\delta$, every adjacent difference changes by at
most $\delta$. Thus a strict single-valley sign pattern with every gap
larger than $\delta$ transfers. For the block partition in
\cref{thm:block}, the left-minus-middle average changes by at most
$\delta\max\{1/\ell,1/m\}$, and the right-minus-middle average by at
most $\delta\max\{1/r,1/m\}$. Positive margins exceeding these bounds
transfer block U-shape. These bounds are sufficient, not necessary.

\subsection{Necessity of effective-mixing and diagonal conditions}
For $n=2$, let $\lambda_t=1$ and
\[
 A^{(t)}=\begin{pmatrix}1&0\\2^{-t-1}&1-2^{-t-1}\end{pmatrix},\qquad t\ge1.
\]
Every state can reach $1$ at every layer, but
$P^{(T)}_{22}=\prod_{t=1}^T(1-2^{-t-1})$ has a positive limit.
Thus $\sum_t\lambda_t=\infty$ and layerwise reachability do not
suffice for collapse. Conversely, one layer with $\lambda_1=1$ and
$A^{(1)}_{22}=0$, followed by $\lambda_t=0$, removes the second diagonal
path despite finite total mixing. This explains the nonzero-factor
condition in \cref{thm:residual_dichotomy}(i).

\section{Fixed-kernel exponential sanity check}
\label{app:infinite-sanity}
\begin{corollary}[Fixed-kernel exponential sanity check]
\label{cor:infinite-exp}
For a fixed causal stochastic $A$, every finite $\tau\ge0$ satisfies
\[
 [e^{\tau(A-I)}]_{ii}=e^{-\tau(1-A_{ii})}>0.
\]
If every state can reach $1$ along positive entries of $A$, then
$e_i^\top e^{\tau(A-I)}\to e_1^\top$ as $\tau\to\infty$.
\end{corollary}
This recovers the finite-mixing retention and infinite-mixing collapse
behavior of the simplified model. It does not identify the finite-$\tau$
exponential approximation with the infinite-depth limit of a trained network.

\begin{proof}
The diagonal of the exponential of a triangular matrix is the exponential
of its diagonal. This gives the finite-$\tau$ equality. For the limit,
reachability gives a finite path from every state to $1$. Since $A$ is
fixed and the state space finite, there are an integer $L\ge1$ and
$\rho>0$ such that $(A^L)_{i1}\ge\rho$ for all $i>1$; shorter paths
can be padded at the absorbing state. The nonabsorbed mass after $k$
discrete steps is at most $(1-\rho)^{\lfloor k/L\rfloor}$.
In the Poisson mixture, $N_\tau\to\infty$ in probability. Averaging
this bound, which is bounded by one and tends to zero with $k$, proves
the limit.
\end{proof}

\section{Kernel construction and aggregation}
\label{app:kernel-constructions}

\paragraph{Content-free ALiBi kernel.}
For head slopes $m_h\ge0$ and causal mask $\mathcal M$, setting content
logits to zero gives
\begin{equation}
 B^{h}_{ij}=\frac{\ind_{\{(i,j)\in\mathcal M\}}\exp[-m_h(i-j)]}
 {\sum_{k:(i,k)\in\mathcal M}\exp[-m_h(i-k)]},\qquad
 A_{ij}=\frac1H\sum_{h=1}^H B^h_{ij}.
 \label{eq:alibi}
\end{equation}
When the slopes, head aggregation, and mask are the same at every layer,
this kernel is layer-independent. Otherwise the layer index is retained.

\paragraph{Forward contribution kernel for RoPE.}
RoPE rotates content-dependent queries and keys. Zeroing those content
vectors removes the rotation-dependent signal as well; a masked-uniform
kernel remains, not an ALiBi-like positional prior. To retain forward
contributions, let $B^h_{t,ij}(x)$ be raw attention for input $x$ and use
column-vector feature notation to define
\begin{equation}
 c_{t,ij}(x)=\sum_h B^h_{t,ij}(x)W^h_{O,t}V_{t,hj}(x),
 \qquad a_{t,i}(x)=\sum_j c_{t,ij}(x).
 \label{eq:sources}
\end{equation}
For bias-free output projections, these are the source vectors of the
attention update. For rows with positive total score, we use an ALTI-style
contribution score~\citep{ferrando2022alti} to scalarize the attention branch:
\begin{equation}
 S_{t,ij}(x)=\bigl[\norm{a_{t,i}(x)}_1-
 \norm{a_{t,i}(x)-c_{t,ij}(x)}_1\bigr]_+,
 \qquad
 A^{(t)}_{ij}(x)=\frac{S_{t,ij}(x)}{\sum_k S_{t,ik}(x)}.
 \label{eq:contribution}
\end{equation}
If $\sum_k S_{t,ik}(x)=0$, we set
$A^{(t)}_{i\cdot}(x)=e_i^\top$.
Equivalently, the effective logits in the setup are
$\ell^{(t)}_{ij}(x)=\log S_{t,ij}(x)$, with $\log 0=-\infty$,
for rows with positive total score. Their softmax gives exactly the
normalized scores above, not the softmax of the scores themselves.
Here $[z]_+=\max\{z,0\}$. Inadmissible entries are zero. Head vectors
are summed before scalarization, so directions and cancellation are retained
at this step. The residual is not included in $c_{t,ij}$; it enters only
through the residual transition in \cref{sec:setup}. The resulting kernel is causal and stochastic,
as required by the rollout theory.

This definition uses no gradient, output-attribution target, or input
blank-out score. It remains dependent on the model, corpus, and chosen
feature norm. No explicit MLP contribution is assigned to the transport
kernel; position-wise processing can nevertheless affect subsequent
attention and measured influence.

\FloatBarrier

\section{Additional rollout and influence analysis}
\label{app:rollout}

This appendix provides additional qualitative and quantitative results for the
controlled rollout experiments discussed in \cref{sec:experiments}. We report
final-token rollout distributions $p^{(T)}$ together with the corresponding
gradient influence proxies $\hat q_{\mathrm{grad}}$ across ALiBi- and
RoPE-based Transformer models.

\subsection{Quantitative agreement between theory and measurement}
\label{app:quantitative-metrics}

To quantify the agreement between the predicted influence distributions
$p^{(T)}$ and the gradient influence proxy
$\hat q_{\mathrm{grad}}$, we report normalized
$1$-Wasserstein distance on token positions, using the ground metric
$d(i,j)=|i-j|/(n-1)$
(\cref{tab:wasserstein_per_dataset}). The $n=256$ results use the
norm-Shapley coefficient from the main experiment for both ALiBi and RoPE.
The adaptive rollout is substantially closer to measured influence than the
attention-only rollout in all 30 model-dataset settings.
We ablate the mixing coefficient and context length separately below.

\begin{table*}[h!]
\centering
\small
\caption{
Normalized $1$-Wasserstein distances at $n=256$ (↓ lower is better) between
controlled rollout variants and gradient influence. Columns (a), (b), and (c)
use $\lambda_t=1$ following
\citet{wu2502}, $\lambda_t=\tfrac12$ following \citet{abnar2005}, and our
adaptive measured norm-Shapley coefficient $\lambda_t$, respectively.
}
\setlength{\tabcolsep}{1pt}
\begin{tabular}{lccc@{\hspace{15pt}}ccc@{\hspace{15pt}}ccc}
\toprule
\multirow{2}{*}{{Model}} &
\multicolumn{3}{c}{{FineWeb-Edu}} &
\multicolumn{3}{c}{{DCLM-Baseline}} &
\multicolumn{3}{c}{{Wikipedia}} \\
\cmidrule(lr){2-4} \cmidrule(lr){5-7} \cmidrule(lr){8-10}
& \makecell{(a) Wu\\$\lambda=1$}
& \makecell{(b) Abnar\\$\lambda=\tfrac12$}
& \makecell{(c) Ours\\adaptive}
& \makecell{(a) Wu\\$\lambda=1$}
& \makecell{(b) Abnar\\$\lambda=\tfrac12$}
& \makecell{(c) Ours\\adaptive}
& \makecell{(a) Wu\\$\lambda=1$}
& \makecell{(b) Abnar\\$\lambda=\tfrac12$}
& \makecell{(c) Ours\\adaptive} \\
\midrule
\texttt{falcon-rw-7b} & 0.66 & 0.63 & \textbf{0.03} & 0.66 & 0.64 & \textbf{0.02} & 0.68 & 0.65 & \textbf{0.05} \\
\texttt{mpt-7b}       & 0.62 & 0.58 & \textbf{0.15} & 0.59 & 0.55 & \textbf{0.19} & 0.66 & 0.62 & \textbf{0.11} \\
\texttt{mpt-30b}      & 0.61 & 0.61 & \textbf{0.02} & 0.58 & 0.58 & \textbf{0.04} & 0.65 & 0.65 & \textbf{0.02} \\
\texttt{bloom-7b1}    & 0.63 & 0.58 & \textbf{0.13} & 0.61 & 0.56 & \textbf{0.15} & 0.67 & 0.61 & \textbf{0.10} \\
\texttt{bloom-176b}   & 0.65 & 0.65 & \textbf{0.11} & 0.63 & 0.63 & \textbf{0.06} & 0.69 & 0.69 & \textbf{0.11} \\
\midrule
\texttt{qwen2.5-7b}      & 0.63 & 0.59 & \textbf{0.12} & 0.59 & 0.55 & \textbf{0.14} & 0.65 & 0.61 & \textbf{0.10} \\
\texttt{qwen2.5-32b}     & 0.59 & 0.59 & \textbf{0.11} & 0.56 & 0.56 & \textbf{0.12} & 0.64 & 0.64 & \textbf{0.07} \\
\texttt{qwen2.5-72b}     & 0.60 & 0.60 & \textbf{0.12} & 0.56 & 0.56 & \textbf{0.10} & 0.62 & 0.62 & \textbf{0.11} \\
\texttt{qwen3-8b-base}   & 0.58 & 0.57 & \textbf{0.09} & 0.57 & 0.56 & \textbf{0.10} & 0.64 & 0.62 & \textbf{0.05} \\
\texttt{smollm3-3b-base} & 0.60 & 0.60 & \textbf{0.11} & 0.60 & 0.59 & \textbf{0.11} & 0.65 & 0.64 & \textbf{0.07} \\
\bottomrule
\end{tabular}
\label{tab:wasserstein_per_dataset}
\end{table*}

\subsection{Qualitative comparison of rollout and measured influence}
\label{app:qualitative-rollout}

Figures~\ref{fig:rollout_comparison_alibi_appendix} and
\ref{fig:rollout_comparison_rope_appendix} extend the central qualitative
comparison. Across the examined models, adaptive residual rollout recovers the
broad, often U-shaped measured profiles, whereas attention-only rollout
concentrates most mass near the first token.

\begin{figure*}[t]
  \centering
  \noindent\makebox[\textwidth][c]{\texttt{bloom-7b1}}\\[0.75ex]
  \begin{subfigure}[t]{0.24\textwidth}
    \centering
    \includegraphics[width=\textwidth]{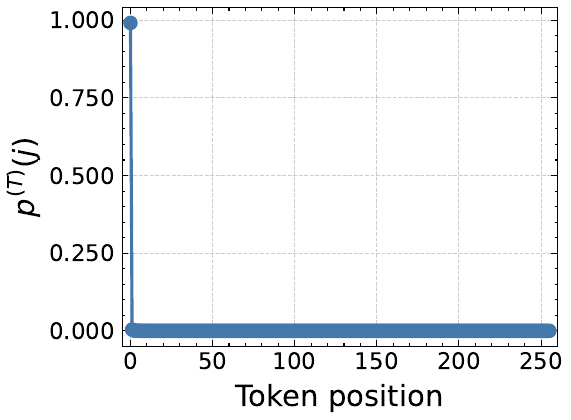}
  \end{subfigure}
  \begin{subfigure}[t]{0.24\textwidth}
    \centering
    \includegraphics[width=\textwidth]{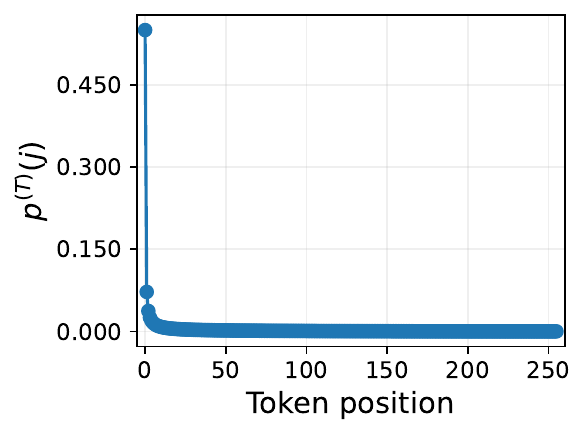}
  \end{subfigure}
  \begin{subfigure}[t]{0.24\textwidth}
    \centering
    \includegraphics[width=\textwidth]{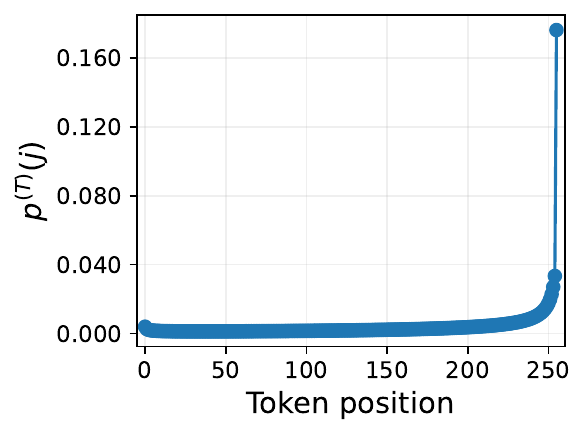}
  \end{subfigure}
  \begin{subfigure}[t]{0.24\textwidth}
    \centering
    \includegraphics[width=\textwidth]{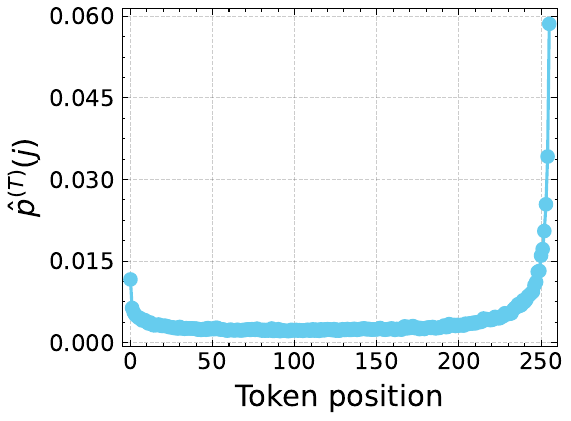}
  \end{subfigure}
  \noindent\makebox[\textwidth][c]{\texttt{mpt-30b}}\\[0.75ex]
  \begin{subfigure}[t]{0.24\textwidth}
    \centering
    \includegraphics[width=\textwidth]{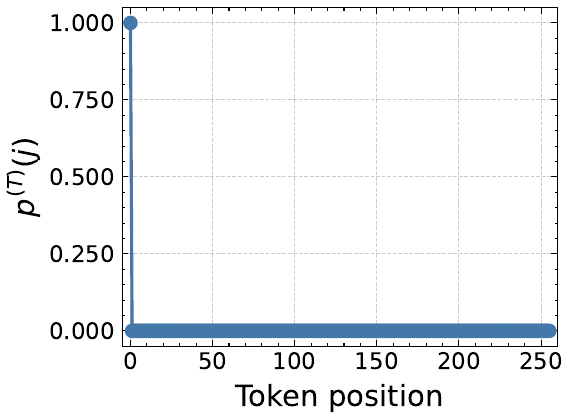}
  \end{subfigure}
  \begin{subfigure}[t]{0.24\textwidth}
    \centering
    \includegraphics[width=\textwidth]{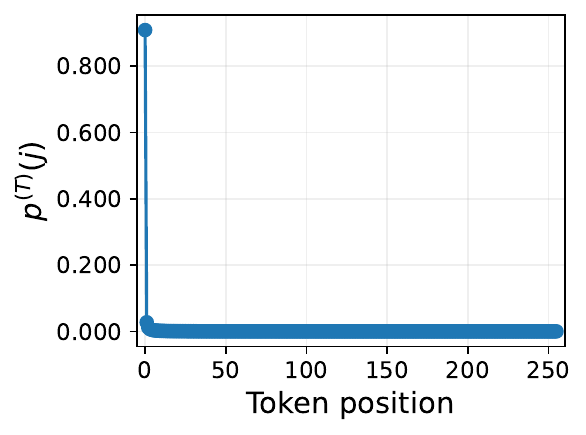}
  \end{subfigure}
  \begin{subfigure}[t]{0.24\textwidth}
    \centering
    \includegraphics[width=\textwidth]{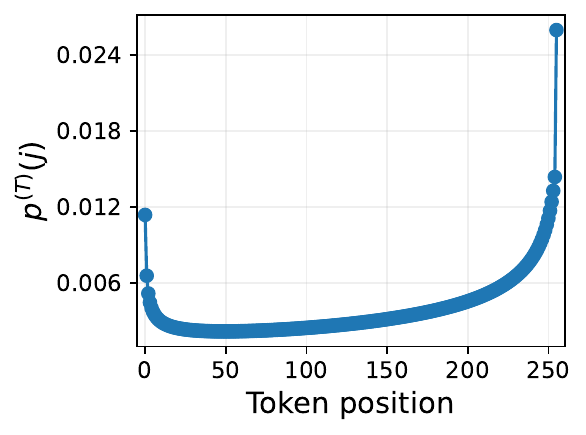}
  \end{subfigure}
  \begin{subfigure}[t]{0.24\textwidth}
    \centering
    \includegraphics[width=\textwidth]{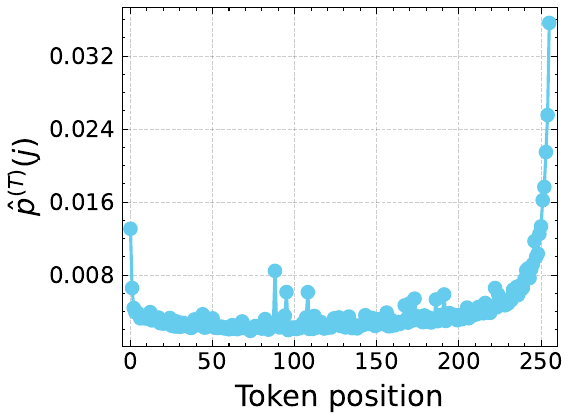}
  \end{subfigure}
  \noindent\makebox[\textwidth][c]{\texttt{mpt-7b}}\\[0.75ex]
  \begin{subfigure}[t]{0.24\textwidth}
    \centering
    \includegraphics[width=\textwidth]{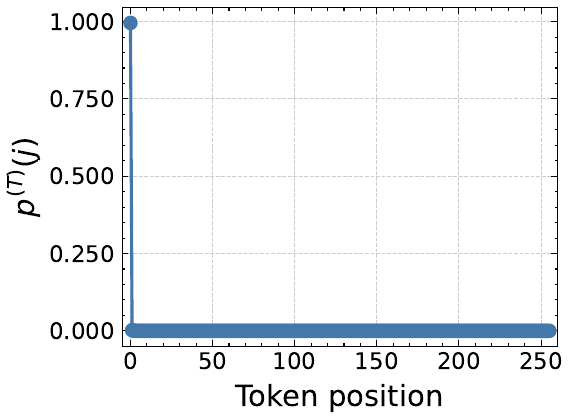}
  \end{subfigure}
  \begin{subfigure}[t]{0.24\textwidth}
    \centering
    \includegraphics[width=\textwidth]{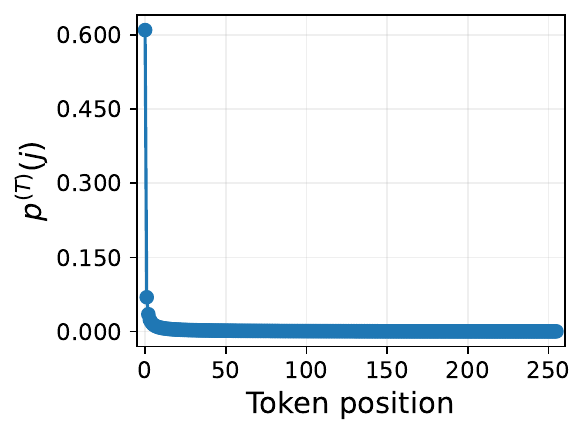}
  \end{subfigure}
  \begin{subfigure}[t]{0.24\textwidth}
    \centering
    \includegraphics[width=\textwidth]{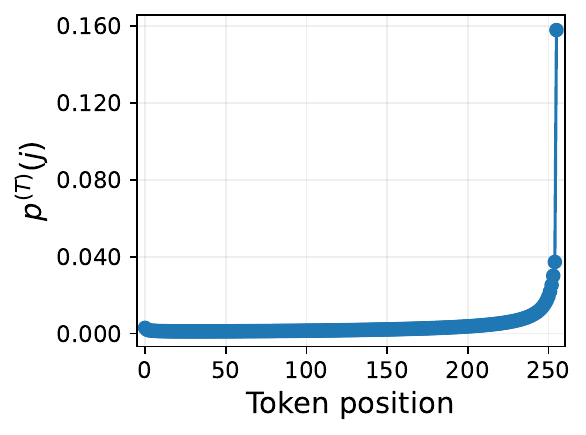}
  \end{subfigure}
  \begin{subfigure}[t]{0.24\textwidth}
    \centering
    \includegraphics[width=\textwidth]{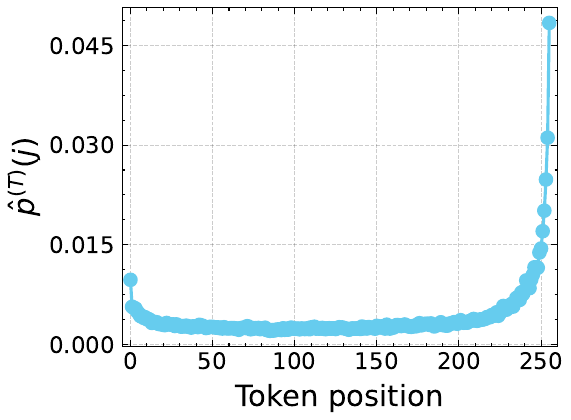}
  \end{subfigure}
  \noindent\makebox[\textwidth][c]{\texttt{bloom-176b}}\\[0.75ex]
  \begin{subfigure}[t]{0.24\textwidth}
    \centering
    \includegraphics[width=\textwidth]{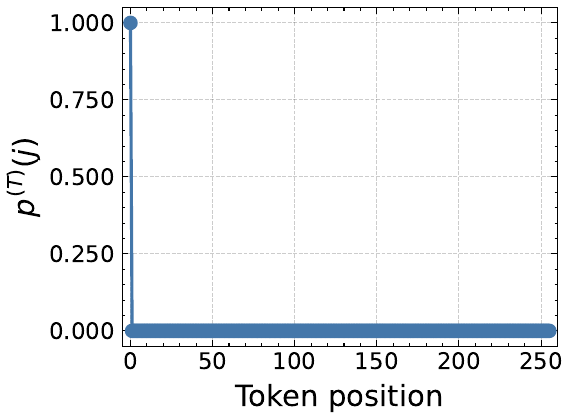}
    \caption{Attention-only rollout ($\lambda_t~=~1$)}
  \end{subfigure}
  \begin{subfigure}[t]{0.24\textwidth}
    \centering
    \includegraphics[width=\textwidth]{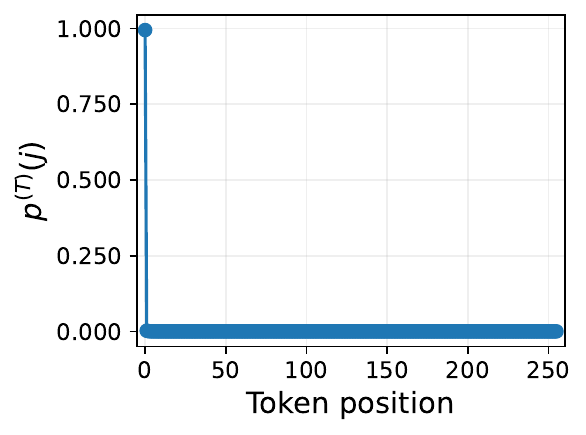}
    \caption{Fixed residual rollout ($\lambda_t~=~\tfrac12$)}
  \end{subfigure}
  \begin{subfigure}[t]{0.24\textwidth}
    \centering
    \includegraphics[width=\textwidth]{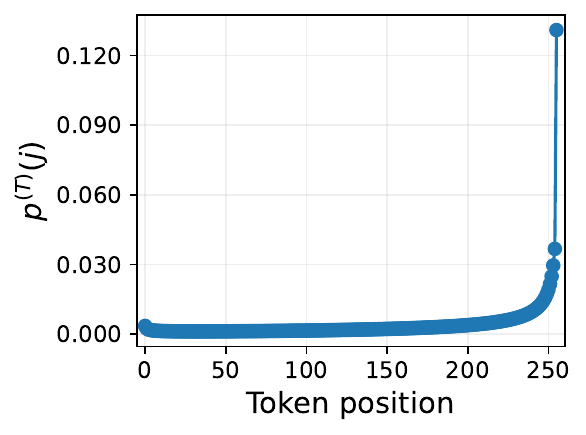}
    \caption{Adaptive residual rollout}
  \end{subfigure}
  \begin{subfigure}[t]{0.24\textwidth}
    \centering
    \includegraphics[width=\textwidth]{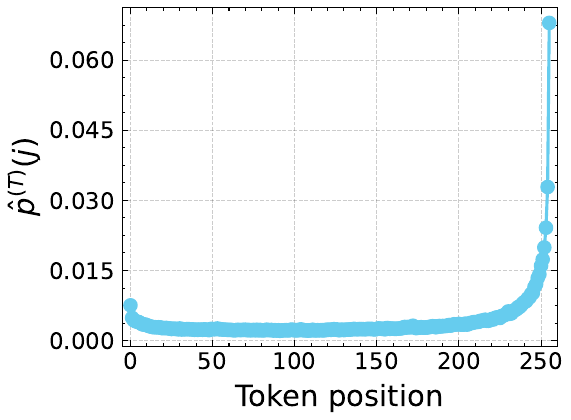}
    \caption{Measured input token influence}
  \end{subfigure}
  \caption{Final-token profiles at sequence length $n=256$ on FineWeb-Edu for
  the four ALiBi models not shown in \cref{fig:rollout_comparison_main}. The
  columns show attention-only rollout, fixed residual rollout, adaptive
  residual rollout, and measured gradient influence.}
  \label{fig:rollout_comparison_alibi_appendix}
\end{figure*}

\begin{figure*}[t]
  \centering
  \noindent\makebox[\textwidth][c]{\texttt{SmolLM3-3B}}\\[0.75ex]
  \begin{subfigure}[t]{0.24\textwidth}
    \includegraphics[width=\textwidth]{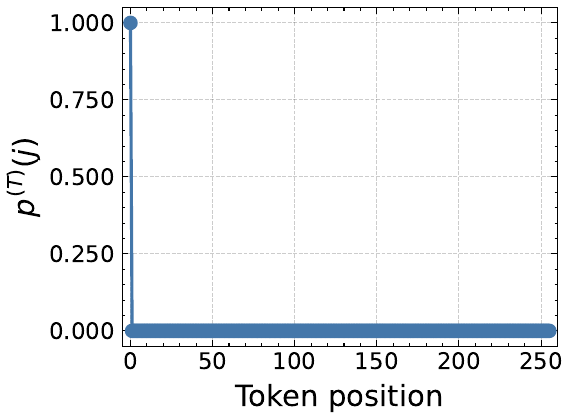}
  \end{subfigure}
  \begin{subfigure}[t]{0.24\textwidth}
    \includegraphics[width=\textwidth]{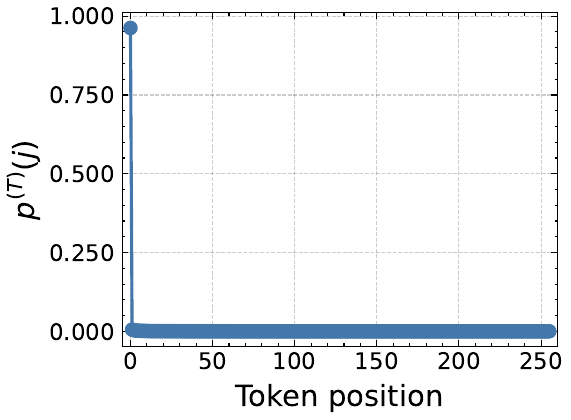}
  \end{subfigure}
  \begin{subfigure}[t]{0.24\textwidth}
    \includegraphics[width=\textwidth]{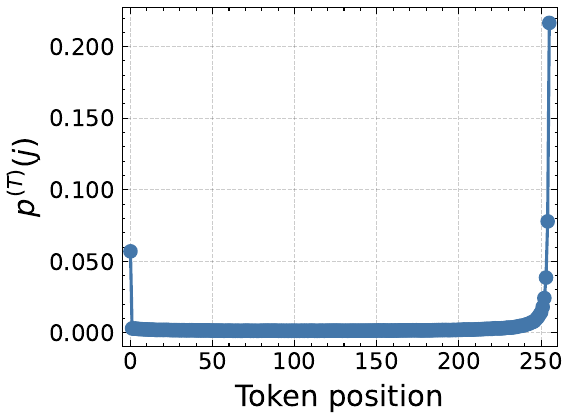}
  \end{subfigure}
  \begin{subfigure}[t]{0.24\textwidth}
    \includegraphics[width=\textwidth]{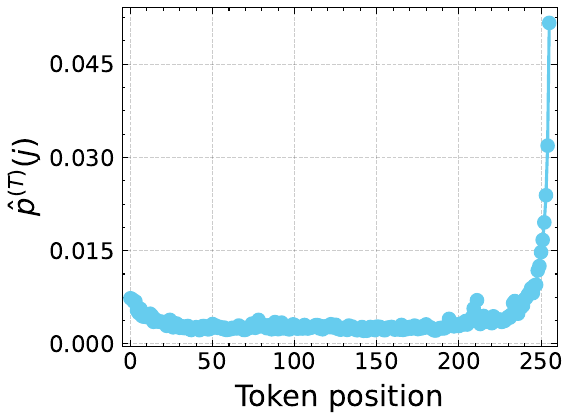}
  \end{subfigure}
  \noindent\makebox[\textwidth][c]{\texttt{Qwen2.5-7B}}\\[0.75ex]
  \begin{subfigure}[t]{0.24\textwidth}
    \includegraphics[width=\textwidth]{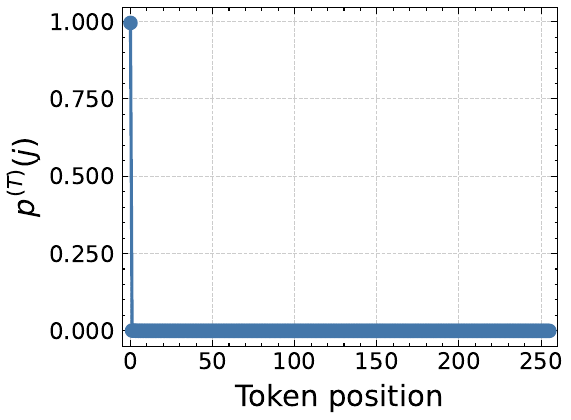}
  \end{subfigure}
  \begin{subfigure}[t]{0.24\textwidth}
    \includegraphics[width=\textwidth]{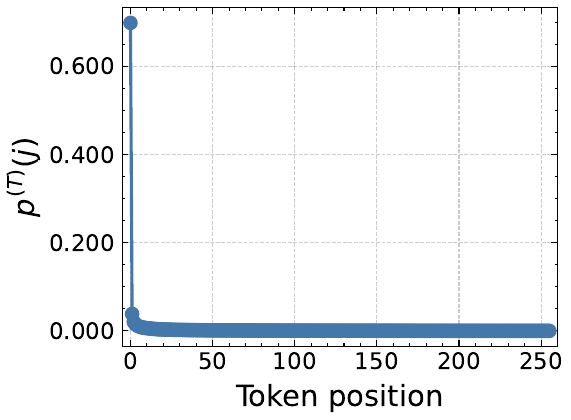}
  \end{subfigure}
  \begin{subfigure}[t]{0.24\textwidth}
    \includegraphics[width=\textwidth]{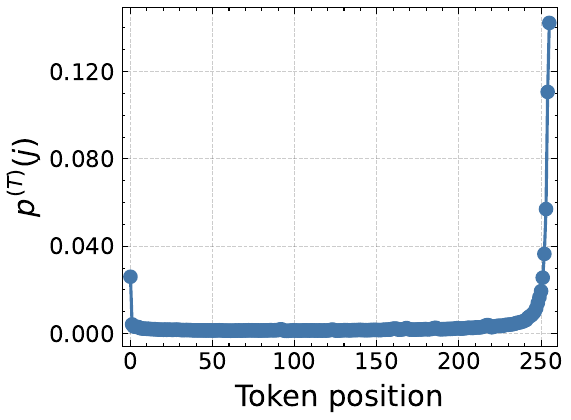}
  \end{subfigure}
  \begin{subfigure}[t]{0.24\textwidth}
    \includegraphics[width=\textwidth]{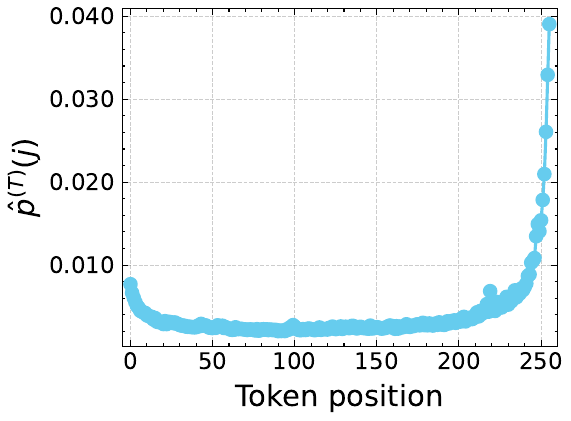}
  \end{subfigure}
  \noindent\makebox[\textwidth][c]{\texttt{Qwen2.5-32B}}\\[0.75ex]
  \begin{subfigure}[t]{0.24\textwidth}
    \includegraphics[width=\textwidth]{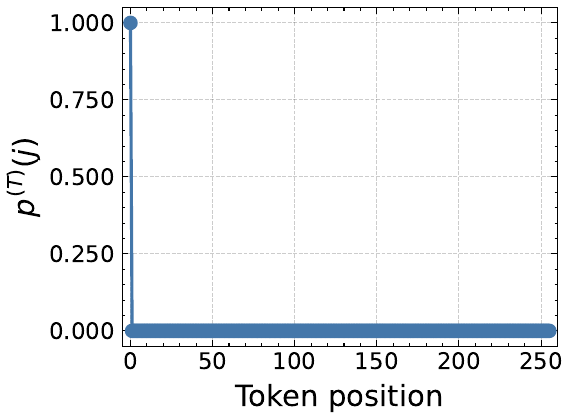}
  \end{subfigure}
  \begin{subfigure}[t]{0.24\textwidth}
    \includegraphics[width=\textwidth]{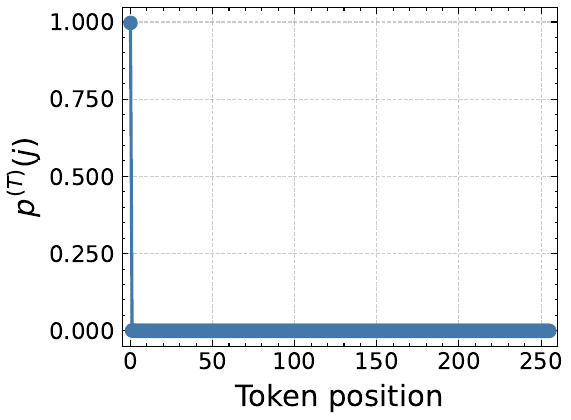}
  \end{subfigure}
  \begin{subfigure}[t]{0.24\textwidth}
    \includegraphics[width=\textwidth]{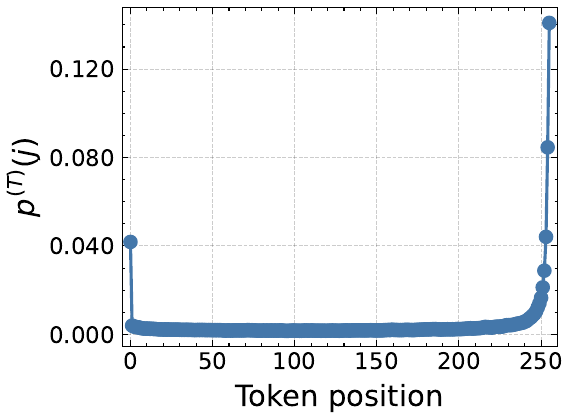}
  \end{subfigure}
  \begin{subfigure}[t]{0.24\textwidth}
    \includegraphics[width=\textwidth]{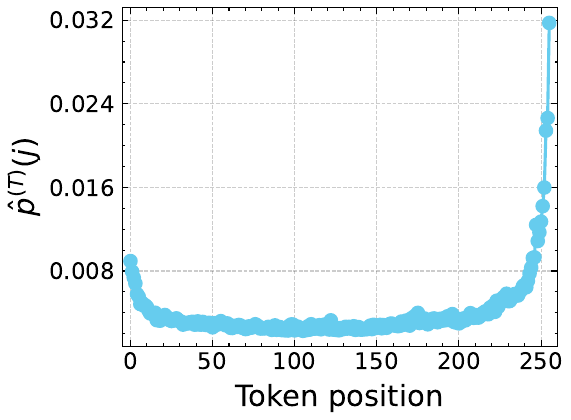}
  \end{subfigure}
  \noindent\makebox[\textwidth][c]{\texttt{Qwen2.5-72B}}\\[0.75ex]
  \begin{subfigure}[t]{0.24\textwidth}
    \includegraphics[width=\textwidth]{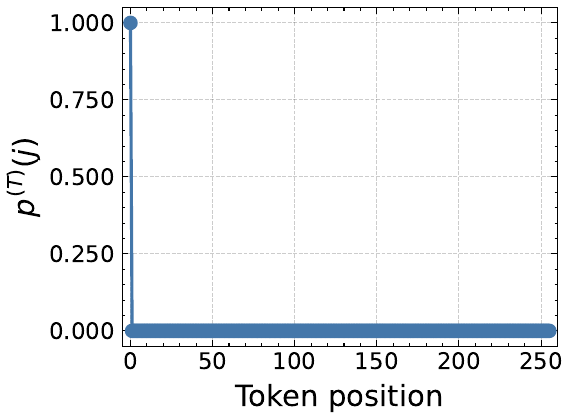}
    \caption{Attention-only rollout ($\lambda_t~=~1$)}
  \end{subfigure}
  \begin{subfigure}[t]{0.24\textwidth}
    \includegraphics[width=\textwidth]{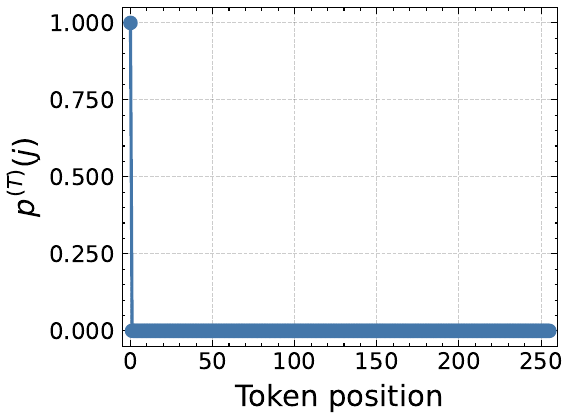}
    \caption{Fixed residual rollout ($\lambda_t~=~\tfrac12$)}
  \end{subfigure}
  \begin{subfigure}[t]{0.24\textwidth}
    \includegraphics[width=\textwidth]{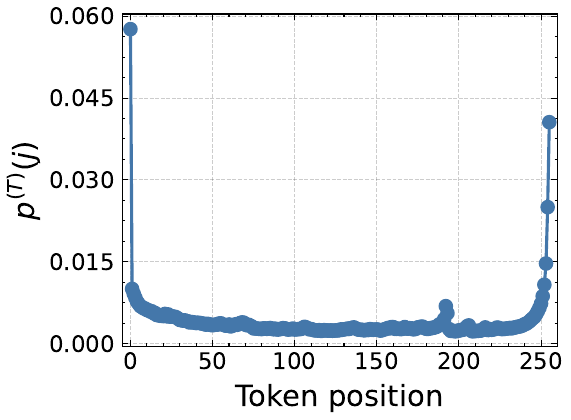}
    \caption{Adaptive residual rollout}
  \end{subfigure}
  \begin{subfigure}[t]{0.24\textwidth}
    \includegraphics[width=\textwidth]{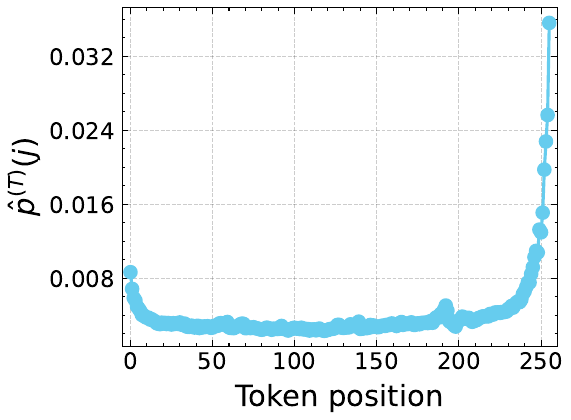}
    \caption{Measured input token influence}
  \end{subfigure}
  \caption{Final-token profiles at sequence length $n=256$ on FineWeb-Edu for
  the four RoPE models not shown in \cref{fig:rollout_comparison_main}. The
  columns show attention-only rollout, fixed residual rollout, adaptive
  residual rollout, and measured gradient influence.}
  \label{fig:rollout_comparison_rope_appendix}
\end{figure*}

\FloatBarrier

\FloatBarrier

\section{Measured residual-mixing schedules}
\label{app:lambda}

\Cref{fig:lambdas-alibi,fig:lambdas-rope} report the depth-wise norm-Shapley
coefficient from \cref{eq:shapley} for the ALiBi and RoPE models,
respectively. These are the measured schedules used by our adaptive rollouts;
neither gradients nor blank-out influence enter their construction.

\begin{figure}[!ht]
  \centering
  \begin{subfigure}{0.49\textwidth}
    \centering
    \includegraphics[width=\linewidth]{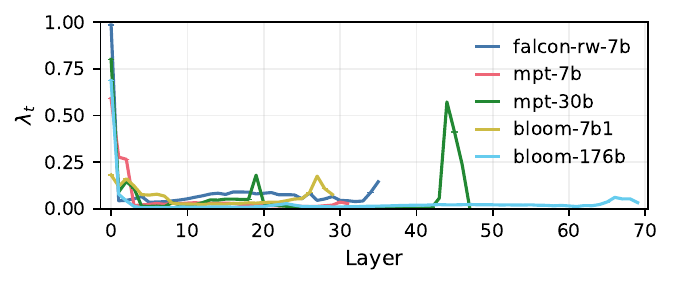}
    \caption{FineWeb-Edu}
  \end{subfigure}

  \vspace{0.5em}

  \begin{subfigure}{0.49\textwidth}
    \centering
    \includegraphics[width=\linewidth]{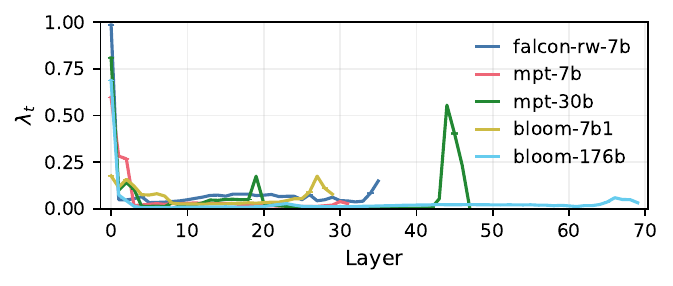}
    \caption{DCLM-Baseline}
  \end{subfigure}\hfill
  \begin{subfigure}{0.49\textwidth}
    \centering
    \includegraphics[width=\linewidth]{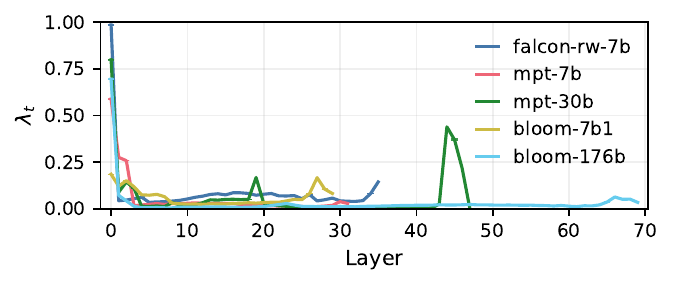}
    \caption{Wikipedia}
  \end{subfigure}

  \caption{Depth-wise norm-Shapley mixing coefficients for the ALiBi models
  across three datasets. Mean $\pm$ 95\% confidence interval at sequence
  length $n=256$; the intervals are narrow relative to
  the plotted range.}
  \label{fig:lambdas-alibi}
\end{figure}

\begin{figure}[!ht]
  \centering
  \begin{subfigure}{0.49\textwidth}
    \centering
    \includegraphics[width=\linewidth]{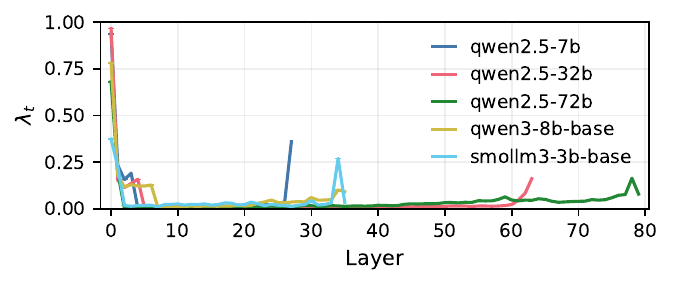}
    \caption{FineWeb-Edu}
  \end{subfigure}

  \vspace{0.5em}

  \begin{subfigure}{0.49\textwidth}
    \centering
    \includegraphics[width=\linewidth]{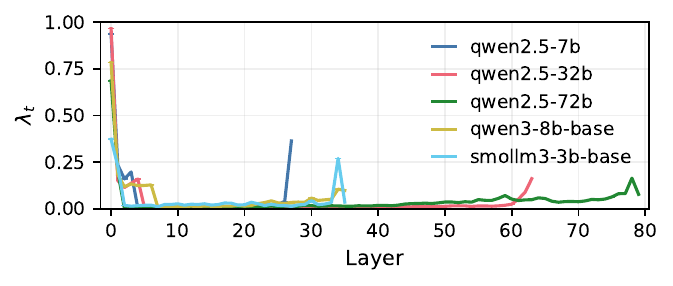}
    \caption{DCLM-Baseline}
  \end{subfigure}\hfill
  \begin{subfigure}{0.49\textwidth}
    \centering
    \includegraphics[width=\linewidth]{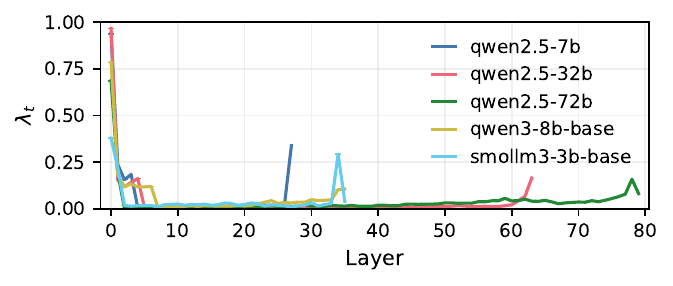}
    \caption{Wikipedia}
  \end{subfigure}

  \caption{Depth-wise norm-Shapley mixing coefficients for the RoPE models
  across three datasets. Mean $\pm$ 95\% confidence interval at sequence
  length $n=256$; the intervals are narrow relative to
  the plotted range.}
  \label{fig:lambdas-rope}
\end{figure}

\FloatBarrier

\FloatBarrier

\section{Ablation studies}
\label{app:ablations}

\subsection{Corpus-level rollout approximations}
\label{app:mean-field-ablation}

We evaluate the two approximations in
\cref{eq:product-of-expectations,eq:mean-field-rollout} separately under the
same norm-Shapley protocol as the main results. Let $P_{\mathrm{input}}$ denote
the average of input-specific rollouts, $P_{\mathrm{joint}}$ the product of
the expected layer transitions, and $P_{\mathrm{mf}}$ the rollout that
additionally averages $\lambda_t$ and $A^{(t)}$ separately within each layer.
Cross-layer compares $P_{\mathrm{input}}$ with
$P_{\mathrm{joint}}$ and tests the product of expectations against the
expectation of the product. Within-layer compares
$P_{\mathrm{joint}}$ with $P_{\mathrm{mf}}$ and tests the mean-field
factorization. Combined compares $P_{\mathrm{input}}$ directly with
$P_{\mathrm{mf}}$. Table~\ref{tab:corpus-approximation-ablation} reports
normalized $1$-Wasserstein distances between their final rows.

\begin{table*}[h!]
\centering
\small
\caption{Corpus-level rollout approximation ablation at $n=256$ under the
norm-Shapley protocol. Entries are normalized final-row $1$-Wasserstein
distances in \textbf{units of $\mathbf{10^{-3}}$} (lower is better).}
\setlength{\tabcolsep}{1pt}
\begin{tabular}{lccc@{\hspace{15pt}}ccc@{\hspace{15pt}}ccc}
\toprule
\multirow{2}{*}{{Model}} &
\multicolumn{3}{c}{{FineWeb-Edu}} &
\multicolumn{3}{c}{{DCLM-Baseline}} &
\multicolumn{3}{c}{{Wikipedia}} \\
\cmidrule(lr){2-4} \cmidrule(lr){5-7} \cmidrule(lr){8-10}
& \makecell{Cross-\\layer}
& \makecell{Within-\\layer}
& Combined
& \makecell{Cross-\\layer}
& \makecell{Within-\\layer}
& Combined
& \makecell{Cross-\\layer}
& \makecell{Within-\\layer}
& Combined \\
\midrule
\texttt{falcon-rw-7b} & 0.17 & 0.00 & 0.17 & 0.31 & 0.00 & 0.31 & 0.20 & 0.00 & 0.20 \\
\texttt{mpt-7b}       & 0.01 & 0.00 & 0.01 & 0.02 & 0.00 & 0.02 & 0.01 & 0.00 & 0.01 \\
\texttt{mpt-30b}      & 0.10 & 0.00 & 0.10 & 0.17 & 0.00 & 0.17 & 0.39 & 0.00 & 0.39 \\
\texttt{bloom-7b1}    & 0.03 & 0.00 & 0.03 & 0.03 & 0.00 & 0.03 & 0.03 & 0.00 & 0.03 \\
\texttt{bloom-176b}   & 0.02 & 0.00 & 0.02 & 0.02 & 0.00 & 0.02 & 0.02 & 0.00 & 0.02 \\
\midrule
\texttt{qwen2.5-7b}      & 0.72 & 0.60 & 0.13 & 3.58 & 0.24 & 3.71 & 3.27 & 0.23 & 3.10 \\
\texttt{qwen2.5-32b}     & 0.07 & 0.04 & 0.04 & 5.27 & 0.41 & 5.48 & 4.36 & 0.35 & 4.27 \\
\texttt{qwen2.5-72b}     & 0.01 & 0.01 & 0.01 & 5.25 & 1.05 & 6.13 & 5.65 & 1.03 & 6.42 \\
\texttt{qwen3-8b-base}   & 0.52 & 0.47 & 0.08 & 5.42 & 0.12 & 5.50 & 5.62 & 0.37 & 5.25 \\
\texttt{smollm3-3b-base} & 0.14 & 0.09 & 0.04 & 4.95 & 0.25 & 4.76 & 4.01 & 0.50 & 3.62 \\
\bottomrule
\end{tabular}
\label{tab:corpus-approximation-ablation}
\end{table*}

Both approximation errors remain small across models, positional mechanisms,
and corpora. The cross-layer error is at most
$5.65\times10^{-3}$ and the within-layer mean-field error is at most
$1.05\times10^{-3}$; their combined effect is at most $6.42\times10^{-3}$.
These discrepancies are well below the distances between rollout and measurement in
\cref{tab:wasserstein_per_dataset}. For the content-free ALiBi construction,
the within-layer error is zero because $A^{(t)}$ is input-independent, so
$\mathbb{E}[\lambda_t A^{(t)}]=\mathbb{E}[\lambda_t]A^{(t)}$ exactly, and the
cross-layer error is at most $3.91\times10^{-4}$. For RoPE, neither
factorization is exact, but all reported discrepancies remain well below the
distances between rollout and measurement. This supports using the tractable mean-field
rollout $P_{\mathrm{mf}}^{(T)}$ as the corpus-level prediction.

\subsection{Context length}
\label{app:iti-2048}

We repeat the analysis at $n=2048$. Figure~\ref{fig:iti_all_models_2048}
shows the gradient influence profiles on Wikipedia, and
Table~\ref{tab:wasserstein_length_2048} reports the rollout comparison on
Wikipedia. This tests whether the conclusions persist at longer context.

\begin{figure*}[t]
  \centering
  \begin{subfigure}[t]{0.19\textwidth}
    \centering
    \includegraphics[width=\textwidth]{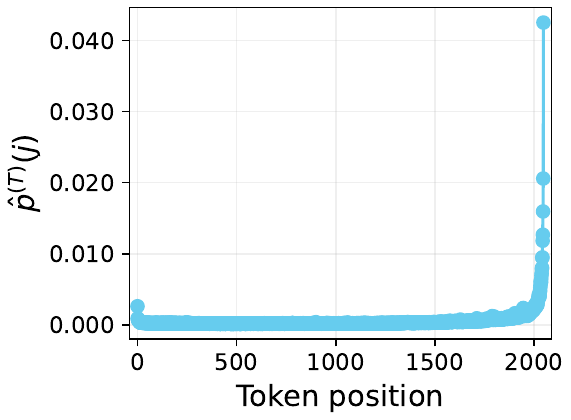}
    \caption{\texttt{bloom-176b}}
  \end{subfigure}
  \begin{subfigure}[t]{0.19\textwidth}
    \centering
    \includegraphics[width=\textwidth]{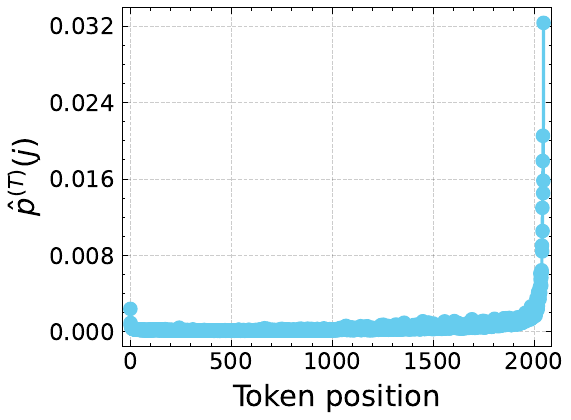}
    \caption{\texttt{bloom-7b1}}
  \end{subfigure}
  \begin{subfigure}[t]{0.19\textwidth}
    \centering
    \includegraphics[width=\textwidth]{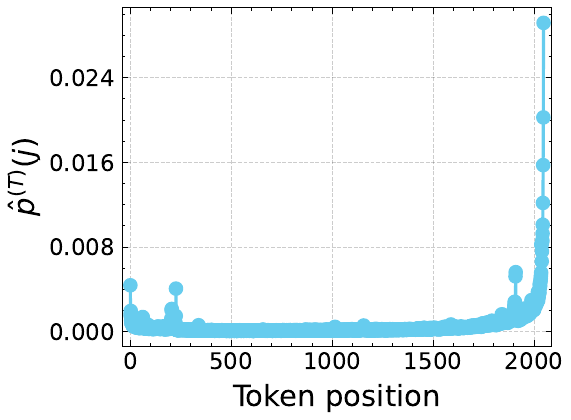}
    \caption{\texttt{mpt-7b}}
  \end{subfigure}
  \begin{subfigure}[t]{0.19\textwidth}
    \centering
    \includegraphics[width=\textwidth]{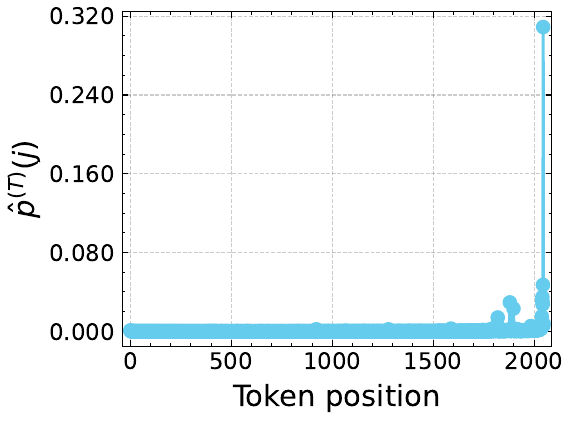}
    \caption{\texttt{mpt-30b}}
  \end{subfigure}
  \begin{subfigure}[t]{0.19\textwidth}
    \centering
    \includegraphics[width=\textwidth]{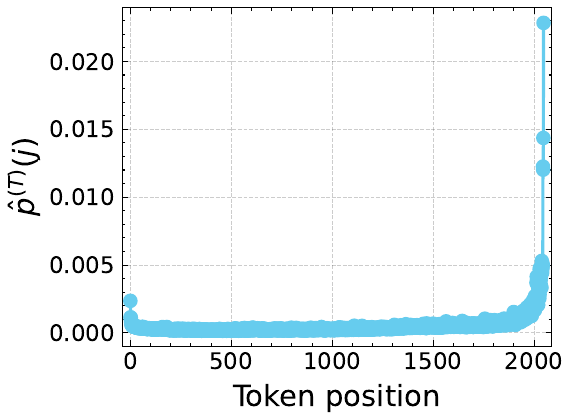}
    \caption{\texttt{falcon-rw-7b}}
  \end{subfigure}

  \medskip

  \begin{subfigure}[t]{0.19\textwidth}
    \centering
    \includegraphics[width=\textwidth]{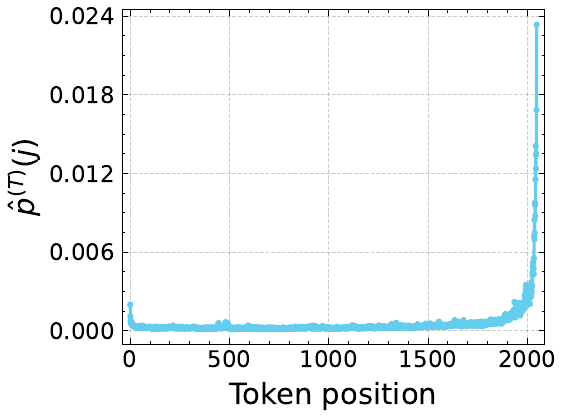}
    \caption{\texttt{qwen2.5-7b}}
  \end{subfigure}
  \hfill
  \begin{subfigure}[t]{0.19\textwidth}
    \centering
    \includegraphics[width=\textwidth]{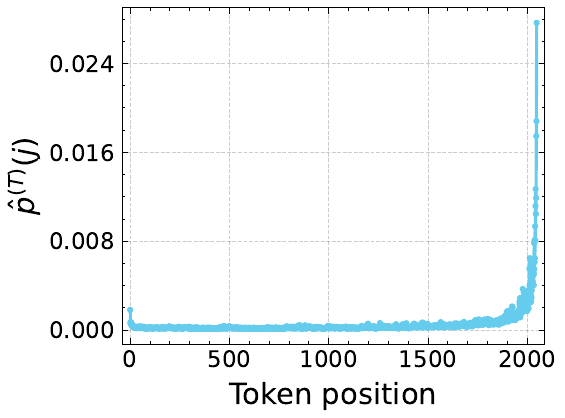}
    \caption{\texttt{qwen2.5-32b}}
  \end{subfigure}
  \hfill
  \begin{subfigure}[t]{0.19\textwidth}
    \centering
    \includegraphics[width=\textwidth]{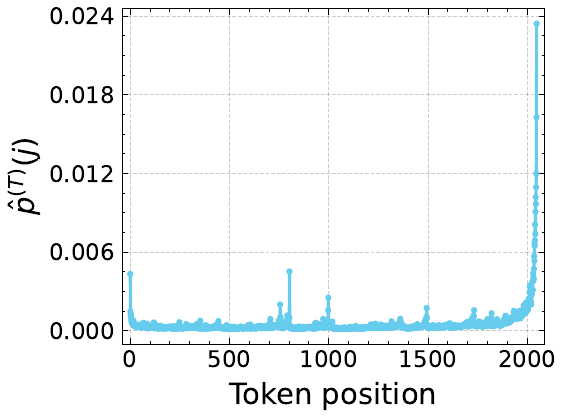}
    \caption{\texttt{qwen2.5-72b}}
  \end{subfigure}
  \hfill
  \begin{subfigure}[t]{0.19\textwidth}
    \centering
    \includegraphics[width=\textwidth]{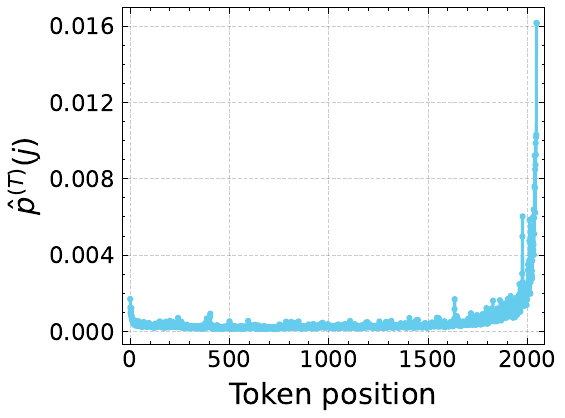}
    \caption{\texttt{qwen3-8b-base}}
  \end{subfigure}
  \hfill
  \begin{subfigure}[t]{0.19\textwidth}
    \centering
    \includegraphics[width=\textwidth]{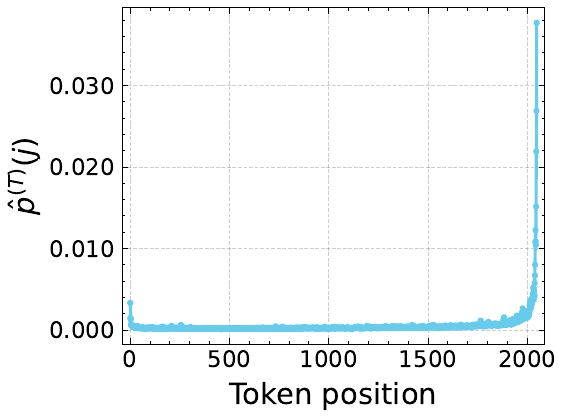}
    \caption{\texttt{smollm3-3b-base}}
  \end{subfigure}
  \caption{Gradient input-token influence $\hat q_{\mathrm{grad}}$ at sequence
  length $n=2048$ on Wikipedia for ALiBi (top) and RoPE (bottom). The profiles
  retain both primacy and recency peaks, although their prominence varies
  across models; in particular, they do not collapse onto the first token at
  the longer context length.}
  \label{fig:iti_all_models_2048}
\end{figure*}

\begin{table*}[h!]
\centering
\small
\caption{Context-length ablation at $n=2048$ on Wikipedia: normalized
$1$-Wasserstein distances (↓ lower is better) to gradient influence.}
\setlength{\tabcolsep}{10pt}
\begin{tabular}{lccc}
\toprule
Model
& \makecell{(a) Wu\\$\lambda=1$}
& \makecell{(b) Abnar\\$\lambda=\tfrac12$}
& \makecell{(c) Ours\\adaptive} \\
\midrule
\texttt{falcon-rw-7b} & 0.46 & 0.20 & \textbf{0.14} \\
\texttt{mpt-7b}       & 0.46 & 0.18 & \textbf{0.16} \\
\texttt{mpt-30b}      & 0.82 & 0.45 & \textbf{0.07} \\
\texttt{bloom-7b1}    & 0.45 & 0.15 & \textbf{0.14}\\
\texttt{bloom-176b}   & 0.74 & 0.49 & \textbf{0.16} \\
\midrule
\texttt{qwen2.5-7b}      & 0.72 & 0.55 & \textbf{0.05} \\
\texttt{qwen2.5-32b}     & 0.76 & 0.74 & \textbf{0.06} \\
\texttt{qwen2.5-72b}     & 0.67 & 0.66 & \textbf{0.12} \\
\texttt{qwen3-8b-base}   & 0.70 & 0.60 & \textbf{0.04} \\
\texttt{smollm3-3b-base} & 0.73 & 0.71 & \textbf{0.09} \\
\bottomrule
\end{tabular}
\label{tab:wasserstein_length_2048}
\end{table*}

The adaptive rollout remains the best method in every comparison at
$n=2048$. Its RoPE distances range from $0.04$ to $0.12$, compared with
$0.67$--$0.76$ for attention-only rollout and $0.55$--$0.74$ for fixed
residual rollout. The ALiBi results retain the same qualitative ordering as at
$n=256$, with adaptive distances between $0.07$ and $0.16$. Their agreement is
somewhat less tight because the controlled ALiBi construction sets all content
logits to zero, increasing the approximation burden over a much longer
sequence. The advantage over fixed mixing is particularly pronounced
for the deeper models. This is consistent with the infinite-depth theory:
fixed positive mixing makes cumulative attention mixing grow with depth and
eventually produces primacy collapse under the theorem's conditions, whereas
adaptive mixing makes collapse depend on cumulative mixing rather than depth
alone and, in these finite models, preserves residual paths to recent tokens.

\FloatBarrier

\subsection{Validation beyond gradient attribution}
\label{app:blankout}

We repeat the comparison using token blank-out~\citep{abnar2005}. For an input
$x$, let $x^{(-j)}$ denote the same input with embedding $j$ set to zero while
positions and the causal mask remain unchanged. Holding the original predicted
token $y(x)$ fixed, we define
\[
b_j(x)=\left|z_{y(x)}(x)-z_{y(x)}(x^{(-j)})\right|,
\qquad
q_{\mathrm{blank}}(j)
=\frac{\mathbb E_{x\sim\mathcal D}[b_j(x)]}
{\sum_k\mathbb E_{x\sim\mathcal D}[b_k(x)]}.
\]
As for gradient influence, raw scores are averaged over the corpus before
normalization across positions.

\begin{table*}[h!]
\centering
\small
\caption{Influence-measurement ablation at $n=256$: normalized
$1$-Wasserstein distances (↓ lower is better) to blank-out influence.
Columns (a), (b), and (c) use $\lambda_t=1$ following
\citet{wu2502}, $\lambda_t=\tfrac12$ following \citet{abnar2005}, and our
adaptive measured norm-Shapley coefficient $\lambda_t$, respectively. The
final column directly compares blank-out and gradient influence.}
\label{tab:blankout-ablation}
\setlength{\tabcolsep}{8pt}
\begin{tabular}{lcccc}
\toprule
Model
& \makecell{(a) Wu\\$\lambda=1$}
& \makecell{(b) Abnar\\$\lambda=\tfrac12$}
& \makecell{(c) Ours\\adaptive}
& \makecell{Blank-out\\vs. gradient} \\
\midrule
\texttt{falcon-rw-7b} & 0.72 & 0.69 & \textbf{0.09} & 0.10 \\
\texttt{mpt-7b}       & 0.66 & 0.62 & \textbf{0.11} & 0.08 \\
\texttt{mpt-30b}      & 0.60 & 0.59 & \textbf{0.04} & 0.03 \\
\texttt{bloom-7b1}    & 0.68 & 0.63 & \textbf{0.08} & 0.06 \\
\texttt{bloom-176b}   & 0.59 & 0.59 & \textbf{0.17} & 0.06 \\
\midrule
\texttt{qwen2.5-7b}      & 0.65 & 0.62 & \textbf{0.09} & 0.06 \\
\texttt{qwen2.5-32b}     & 0.65 & 0.65 & \textbf{0.06} & 0.06 \\
\texttt{qwen2.5-72b}     & 0.67 & 0.67 & \textbf{0.20} & 0.08 \\
\texttt{qwen3-8b-base}   & 0.61 & 0.59 & \textbf{0.08} & 0.03 \\
\texttt{smollm3-3b-base} & 0.61 & 0.61 & \textbf{0.11} & 0.03 \\
\bottomrule
\end{tabular}
\end{table*}

All entries use FineWeb-Edu. The adaptive rollout is closer to blank-out than
either baseline in every case. The blank-out and gradient profiles
have a mean distance of $0.06$, and replacing gradient influence by blank-out
preserves the ordering of all three rollout variants. Across the ten models,
the mean adaptive distance to blank-out is $0.10$, compared with
$0.64$ for attention-only rollout and $0.63$ for fixed residual rollout; the
adaptive distances are also on the same scale as the corresponding
gradient-based results in \cref{tab:wasserstein_per_dataset}. Thus the result
is not specific to gradient attribution or gradient saturation. Moreover, the
adaptive error of $0.10$ is of the same order as the $0.06$ disagreement
between the two influence measures, placing the rollout within the empirical
measurement-sensitivity scale of the target profiles. This proxy disagreement
provides a calibration for the remaining discrepancy between rollout and measurement,
rather than a formal lower bound on it.

\FloatBarrier

\subsection{Predicting an intervention on residual mixing}
\label{app:mixing-intervention}

The rollout makes an intervention prediction that can be tested without
fitting to the resulting influence profile.  For each model, we first record
the unmodified attention and residual branches, the corpus-level transport
kernels, and the norm-Shapley mixing schedule.  We then multiply every
attention-branch output by the fixed factor $\beta$, leaving the residual
branch unchanged.  The original prompts and next-token targets are held fixed
across conditions.  Neither the transport kernels nor the mixing schedule is
remeasured after intervention.

Let $r_t$ and $a_t$ denote the residual and attention branches recorded in the
unmodified model.  The norm-Shapley definition predicts the intervened
coefficient
\begin{equation}
\lambda_t(\beta)
=
\frac{\|r_t+\beta a_t\|+\beta\|a_t\|-\|r_t\|}
{2\|r_t+\beta a_t\|}.
\label{eq:beta-intervention}
\end{equation}
We substitute the corpus average of \cref{eq:beta-intervention} into the
rollout while keeping the unmodified corpus-level kernels fixed.  Thus the
prediction uses only forward quantities measured before intervention; no
gradient influence from an intervened model enters its construction.  We use
the moderate interventions $\beta=0.8$ and $\beta=1.2$ on three ALiBi and
three RoPE models, using $1{,}024$ matched prompts per model.  Prefix and
recency mass refer to the first and last $26$ positions, respectively.

\begin{table*}[t]
\centering
\scriptsize
\caption{Predicted and measured changes under attention-branch scaling,
relative to the unmodified model ($\beta=1$). Changes are percentage points
of normalized prefix or recency mass. Brackets give paired $95\%$ bootstrap
confidence intervals for measured changes.}
\label{tab:beta-intervention-effects}
\setlength{\tabcolsep}{3.2pt}
\begin{tabular}{lc@{\hspace{7pt}}rr@{\hspace{9pt}}rr}
\toprule
& & \multicolumn{2}{c}{Prefix mass} & \multicolumn{2}{c}{Recency mass} \\
\cmidrule(lr){3-4}\cmidrule(lr){5-6}
Model & $\beta$ & Predicted & Measured [95\% CI] & Predicted & Measured [95\% CI] \\
\midrule
\texttt{bloom-7b1}    & 0.8 & $-1.21$ & $-0.74$ [$-0.89,-0.58$] & $+8.44$ & $+2.31$ [$+1.99,+2.62$] \\
                       & 1.2 & $+1.41$ & $+0.78$ [$+0.69,+0.88$] & $-7.47$ & $-2.12$ [$-2.37,-1.84$] \\
\texttt{mpt-7b}       & 0.8 & $-1.04$ & $-0.51$ [$-1.39,+0.53$] & $+8.92$ & $+1.61$ [$+0.46,+2.81$] \\
                       & 1.2 & $+1.08$ & $+1.41$ [$+0.59,+2.39$] & $-7.14$ & $-4.62$ [$-5.89,-3.46$] \\
\texttt{falcon-rw-7b} & 0.8 & $-2.12$ & $-0.79$ [$-1.10,-0.44$] & $+5.99$ & $+0.17$ [$-0.48,+0.83$] \\
                       & 1.2 & $+2.34$ & $+0.96$ [$+0.70,+1.23$] & $-4.95$ & $-0.71$ [$-1.37,-0.04$] \\
\midrule
\texttt{qwen2.5-7b}   & 0.8 & $-1.57$ & $-1.20$ [$-1.99,-0.63$] & $+5.30$ & $+2.21$ [$+1.07,+3.41$] \\
                       & 1.2 & $+1.54$ & $+0.90$ [$+0.20,+1.39$] & $-4.67$ & $-1.66$ [$-2.98,-0.35$] \\
\texttt{qwen3-8b-base}& 0.8 & $-2.65$ & $-1.09$ [$-1.87,-0.47$] & $+7.02$ & $+2.28$ [$+1.32,+3.30$] \\
                       & 1.2 & $+2.68$ & $+0.37$ [$-0.38,+0.94$] & $-6.00$ & $-1.85$ [$-3.09,-0.65$] \\
\texttt{smollm3-3b-base}& 0.8 & $-2.84$ & $-0.71$ [$-1.02,-0.43$] & $+9.55$ & $+0.91$ [$+0.03,+1.98$] \\
                       & 1.2 & $+2.94$ & $+0.44$ [$+0.15,+0.71$] & $-8.37$ & $-0.41$ [$-1.31,+0.76$] \\
\bottomrule
\end{tabular}
\end{table*}

The prefix prediction has the correct sign in all twelve model-intervention
settings.  The complementary recency prediction also has the correct point
estimate in all twelve settings, although its paired interval excludes zero
in ten.  The prefix interval likewise excludes zero in ten of twelve
settings.  The
agreement is not a restatement of the static profile comparison: the
coefficient in \cref{eq:beta-intervention} and the fixed rollout predict the
direction of change before observing the intervened influence curve.
Measured changes are generally smaller than predicted. The rollout isolates
the direct transport effect while holding the original kernels fixed, whereas
the intervention also changes downstream representations and subsequent
attention computations. These unmodeled feedbacks can attenuate the response
without reversing its direction. Accordingly, all 24 point estimates have the
predicted sign, while four confidence intervals include zero because the
corresponding measured shifts are small.

\begin{table}[t]
\centering
\small
\caption{Intervention stability. KL is the mean next-token
$D_{\mathrm{KL}}(p_\beta\|p_1)$; Top-1 is the percentage of prompts retaining
the unmodified next-token prediction.}
\label{tab:beta-intervention-stability}
\setlength{\tabcolsep}{5pt}
\begin{tabular}{lrrrr}
\toprule
& \multicolumn{2}{c}{$\beta=0.8$} & \multicolumn{2}{c}{$\beta=1.2$} \\
\cmidrule(lr){2-3}\cmidrule(lr){4-5}
Model & KL & Top-1 & KL & Top-1 \\
\midrule
\texttt{bloom-7b1}      & 0.03 & 91\% & 0.02 & 93\% \\
\texttt{mpt-7b}         & 0.22 & 83\% & 0.08 & 85\% \\
\texttt{falcon-rw-7b}   & 0.09 & 88\% & 0.05 & 89\% \\
\texttt{qwen2.5-7b}     & 0.08 & 88\% & 0.03 & 92\% \\
\texttt{qwen3-8b-base}  & 0.05 & 90\% & 0.03 & 94\% \\
\texttt{smollm3-3b-base}& 0.08 & 88\% & 0.06 & 91\% \\
\bottomrule
\end{tabular}
\end{table}

The interventions preserve the original top prediction on $83\%$--$94\%$
of prompts, and the mean next-token KL remains at most $0.22$.  These
diagnostics support interpreting $\beta=0.8$ and $\beta=1.2$ as moderate
perturbations around the trained models.  Larger factors are not used in this
analysis because applying the same scale at every layer compounds with depth
and can move the network outside this local regime.

\FloatBarrier
\subsection{Mixing coefficient}
\label{app:mixing-ablation}

We replace the norm-Shapley coefficient by the alternative relative-norm
coefficient, computed using the Frobenius norm,
\[
\lambda_t^{\mathrm{rel}}(x)
:=\frac{\|\mathrm{Attn}(X^{(t-1)})\|_F}
{\|X^{(t-1)}\|_F+
 \|\mathrm{Attn}(X^{(t-1)})\|_F}.
\]

\begin{table*}[h!]
\centering
\small
\caption{Mixing-coefficient ablation at $n=256$: normalized $1$-Wasserstein
distances (↓ lower is better) to gradient influence on FineWeb-Edu.
Columns (a) and (b) use $\lambda_t=1$ following \citet{wu2502} and
$\lambda_t=\tfrac12$ following \citet{abnar2005}. Columns (c) and (d) use our
adaptive rollout with the relative-norm and norm-Shapley coefficients,
respectively. Bold indicates the lowest distance in each row.}
\setlength{\tabcolsep}{7pt}
\begin{tabular}{lcccc}
\toprule
Model
& \makecell{(a) Wu\\$\lambda=1$}
& \makecell{(b) Abnar\\$\lambda=\tfrac12$}
& \makecell{(c) Ours\\relative norm}
& \makecell{(d) Ours\\norm-Shapley} \\
\midrule
\texttt{falcon-rw-7b} & 0.66 & 0.63 & 0.19 & \textbf{0.03} \\
\texttt{mpt-7b}       & 0.62 & 0.58 & \textbf{0.09} & 0.15 \\
\texttt{mpt-30b}      & 0.61 & 0.61 & 0.16 & \textbf{0.02} \\
\texttt{bloom-7b1}    & 0.63 & 0.58 & \textbf{0.04} & 0.13 \\
\texttt{bloom-176b}   & 0.65 & 0.65 & \textbf{0.08} & 0.11 \\
\midrule
\texttt{qwen2.5-7b}      & 0.63 & 0.59 & 0.29 & \textbf{0.12} \\
\texttt{qwen2.5-32b}     & 0.59 & 0.59 & 0.38 & \textbf{0.11} \\
\texttt{qwen2.5-72b}     & 0.60 & 0.60 & 0.48 & \textbf{0.12} \\
\texttt{qwen3-8b-base}   & 0.58 & 0.57 & 0.24 & \textbf{0.09} \\
\texttt{smollm3-3b-base} & 0.60 & 0.60 & 0.42 & \textbf{0.11} \\
\midrule
\textbf{Mean}            & 0.62 & 0.60 & 0.24 & \textbf{0.10} \\
\bottomrule
\end{tabular}
\label{tab:wasserstein_relative_256}
\end{table*}

Norm-Shapley reduces the mean distance from $0.24$ to $0.10$ relative to the
relative-norm calibration and is better for every RoPE model. The comparison
is mixed within the ALiBi family, but norm-Shapley provides the stronger common
calibration across both positional mechanisms while accounting for
reinforcement and cancellation between the two residual branches.

\FloatBarrier


\section{Reproducibility details}
\label{app:reproducibility}

\paragraph{Protocol.}
All experiments evaluate public pretrained checkpoints; no model is trained or
fine-tuned, and no rollout parameter is fitted to the reported influence
profiles. We use seed 42 and batch size one. Samples are taken in deterministic
streaming order from the dataset configurations and splits listed in the
supplementary manifest. Each text is tokenized with the corresponding model
tokenizer, truncated to the required length, and retained only when it contains
exactly $n$ tokens. Models run in evaluation mode in bfloat16, while accumulated
profiles, kernels, and reported distances use double precision. Gradient
influence is the input-embedding gradient norm of the model's predicted
next-token probability. Normalized $1$-Wasserstein distance uses ground cost
$|i-j|/(n-1)$.

\paragraph{Uncertainty.}
The main Wasserstein tables report point estimates between corpus-mean profiles;
there is no variation from training or parameter initialization. The bands in
Appendix~\ref{app:lambda} are pointwise 95\% intervals computed as the sample
mean plus or minus $1.96$ standard errors over prompts. The intervention study
uses the same prompts and targets in every condition and reports paired
percentile-bootstrap intervals from 1{,}000 prompt resamples with seed 42. Its
directional result holds for all 24 reported point estimates. The blank-out
study provides an additional measurement-proxy check independent of gradients.

\paragraph{Compute and artifacts.}
The experiments require inference and attribution only. They were run with
bfloat16 model weights on NVIDIA RTX PRO 6000 Blackwell GPUs with 96\,GB of
device memory, using one GPU for the smaller models and model parallelism for
the largest checkpoints, up to two GPUs for Qwen2.5-72B and four for
BLOOM-176B. Aggregation, Wasserstein evaluation, bootstrap resampling, and
plotting run on CPU. The complete code is available at
\url{https://github.com/ml-uhh/position-bias}. It contains
the configuration files, locked software environment, and commands required to
reproduce the experiments, tables, and figures. It also specifies the exact
model and dataset identifiers, sample counts, sequence lengths, and random
seed. Model weights and datasets are downloaded from their public repositories.


\clearpage
\section{Licenses and usage notes}
\label{app:licenses}

\Cref{tab:licenses} summarizes license names, links, and citations for the models and datasets used in this paper.

\begin{table}[htbp]
\small
\centering
\caption{License summary for models and datasets used in this work.}
\label{tab:licenses}
\setlength{\tabcolsep}{5pt}
\begin{tabular}{@{}m{0.22\textwidth} m{0.20\textwidth} m{0.20\textwidth} m{0.20\textwidth}@{}}
\toprule
Resource & License & Link & Notes / Citation \\
\midrule
Falcon & Apache 2.0 & \href{https://huggingface.co/tiiuae/falcon-rw-7b}{Model repository} & \citep{penedo2306} \\
MPT & Apache 2.0 & \href{https://huggingface.co/eluzhnica/mpt-30b-peft-compatible}{Model repository} & \citep{mosaicml2305,mosaicml2306} \\
BLOOM & BigScience RAIL License v1.0 & \href{https://huggingface.co/bigscience/bloom}{Model repository} & \citep{bigscience2205} \\
Qwen2.5 (7B, 32B) & Apache 2.0 & \href{https://huggingface.co/Qwen/Qwen2.5-32B}{Model repository} & \citep{qwenteam2409} \\
Qwen2.5 (72B) & Qwen License Agreement & \href{https://huggingface.co/Qwen/Qwen2.5-72B}{Model repository} & \citep{qwenteam2409} \\
Qwen3-8B-Base & Apache 2.0 & \href{https://huggingface.co/Qwen/Qwen3-8B-Base}{Model repository} & \citep{qwenteam2505} \\
SmolLM3-3B-Base & Apache 2.0 & \href{https://huggingface.co/HuggingFaceTB/SmolLM3-3B-Base}{Model repository} & \citep{huggingfacetb2507} \\
\midrule
FineWeb-Edu & ODC-By-1.0 & \href{https://huggingface.co/datasets/HuggingFaceFW/fineweb-edu}{Dataset repository} & \citep{dataset-finewebedu} \\
DCLM-Baseline & CC-by-4.0 & \href{https://huggingface.co/datasets/mlfoundations/dclm-baseline-1.0}{Dataset repository} & \citep{dataset-dclm} \\
Wikipedia & CC-by-SA-3.0 & \href{https://huggingface.co/datasets/wikimedia/wikipedia}{Dataset repository} & \citep{dataset-wikipedia} \\
\bottomrule
\end{tabular}
\end{table}


\end{document}